\documentclass[lettersize,journal]{IEEEtran}
\usepackage{amsmath,amsfonts}
\usepackage{algorithm}
\usepackage{amsthm}
\usepackage{algpseudocode}
\usepackage{array}
\usepackage[caption=false,font=normalsize,labelfont=sf,textfont=sf]{subfig}
\usepackage{textcomp}
\usepackage{stfloats}
\usepackage{xcolor}
\usepackage{cite}
\usepackage{url}

\usepackage{multirow}
\usepackage{verbatim}
\usepackage{graphicx}
\def\BibTeX{{\rm B\kern-.05em{\sc i\kern-.025em b}\kern-.08em
    T\kern-.1667em\lower.7ex\hbox{E}\kern-.125emX}}
\usepackage{balance}
\usepackage[normalem]{ulem}
\graphicspath{{fig/}}
\newtheorem{definition}{Definition}
\newtheorem{theorem}{Theorem}

\newtheorem{lemma}{Lemma}

\newtheorem{remark}{Remark}

\begin{document}
\title{Large-Scale Traffic Signal Control Using Constrained Network Partition and Adaptive Deep Reinforcement Learning}
\author{Hankang Gu, Shangbo Wang,
~\IEEEmembership{Member,~IEEE,}
Xiaoguang Ma, Dongyao Jia,\\ Guoqiang Mao, ~\IEEEmembership{Fellow,~IEEE,} Eng Gee Lim, ~\IEEEmembership{Senior Member,~IEEE,} Cheuk Pong Ryan Wong
\thanks{This work was supported in part by XJTLU Postgraduate Research Scholarship (FOSA2106053) and in part by Xi’an Jiaotong-Liverpool University under Research Development Fund RDF-21-02-015.}
\thanks{Hankang Gu is with the School of Advanced Technology, Xi'an Jiaotong-Liverpool University, Suzhou 215123, China, also with the Department of Computer Science, University of Liverpool,
L69 3GJ Liverpool, U.K.(email: Hankang.Gu16@student.xjtlu.edu.cn)}
\thanks{Shangbo Wang, Dongyao Jia, Eng Gee Lim are with the School of Advanced Technology, Xi'an Jiaotong-Liverpool University, Suzhou 215123, China, (e-mail: Shangbo.Wang@xjtlu.edu.cn; Dongyao.Jia@xjtlu.edu.cn; Enggee.Lim@xjtlu.edu.cn)}
                 
\thanks{Xiaoguang Ma is with the College of Information Science and Engineering, Northeastern University, Shenyang 110819, China (email: maxg@mail.neu.edu.cn)
}
\thanks{Guoqiang Mao is with the ISN State Key Lab, Xidian University, Xi'an 710126, China (email:gqmao@xidian.edu.cn)}
\thanks{Cheuk Pong Ryan Wong is with the Department of Civil Engineering,
The University of Hong Kong (email:cpwryan@hku.hk)}
}
\maketitle
\begin{abstract}
Multi-agent Deep Reinforcement Learning (MADRL) based traffic signal control becomes a popular research topic in recent years.
To alleviate the scalability issue of completely centralized RL techniques and the non-stationarity issue of completely decentralized RL techniques on large-scale traffic networks,  
some literature utilizes a regional control approach where the whole network is firstly partitioned into multiple disjoint regions, followed by applying the centralized RL approach to each region.
However, the existing partitioning rules either have no constraints on the topology of regions or require the same topology for all regions. Meanwhile, no existing regional control approach explores the performance of optimal joint action in an exponentially growing regional action space when intersections are controlled by 4-phase traffic signals (EW, EWL, NS, NSL).
In this paper, we propose a novel RL training framework named RegionLight to tackle the above limitations. Specifically, the topology of regions is firstly constrained to a star network which comprises one center and an arbitrary number of leaves. Next, the network partitioning problem is modeled as an optimization problem to minimize the number of regions. Then, an Adaptive Branching Dueling Q-Network (ABDQ) model is proposed to decompose the regional control task into several joint signal control sub-tasks corresponding to particular intersections. Subsequently, these sub-tasks maximize the regional benefits cooperatively. Finally, the global control strategy for the whole network is obtained by concatenating the optimal joint actions of all regions.
Experimental results demonstrate the superiority of our proposed framework over all baselines under both real and synthetic datasets in all evaluation metrics.

\end{abstract}
\section{Introduction}
\IEEEPARstart{T}{raffic} congestion is becoming a significant problem that lead to both financial costs and environmental damage. 
According to a recent study, traffic congestion costs \pounds 595 and 73 hours per driver in the UK \cite{inrix2021uk} while drivers in the USA spent \$ 564 each and wasted 3.4 billion hours a year in total \cite{inrix2021us}.
Meanwhile, the gas emission caused by traffic congestion is now an unignorable contributor to the pollutants responsible for air pollution \cite{who}.
Therefore, there is an urgent need to apply effective strategies to relieve urban traffic congestion.  

Traffic signal control (TSC) is an efficient and direct way to reduce traffic congestion by managing and regulating the movements of vehicles\cite{roess2004traffic}. Existing TSC strategies can be generally classified into two categories: classical methods and AI (Artificial Intelligence) -based methods. Classical methods take rule-based signal plans such as Webster \cite{webster1958traffic,koonce2008traffic} and MaxBand \cite{little1981maxband} which compute optimal signal plans based on traffic parameters such as traffic demand and saturation rate. However, most classical methods assume that traffic flow is uniform and traffic signals share the same cycle length and they hardly adapt to more complex traffic dynamics in real scenarios. Therefore, some adaptive rule-based methods such as Max-pressure\cite{varaiya2013max} and self-organizing traffic lights (SOTL)\cite{gershenson2004self,
cools2013self,long2022traffic} have been proposed to control traffic signals based on real-time information. Inspired by nature, meta-heuristics methods have been applied to solve traffic signal control problems in an evolutionary and long-sighted manner\cite{shaikh2020review,gao2018solving}. Among various categories in meta-heuristics methods, genetic algorithms\cite{mitchell1998introduction}, artificial bee colony algorithms\cite{karaboga2005idea}, and harmony search algorithms\cite{geem2001new} are three common population-based methods and they have been successfully applied to traffic networks with various scales and demands\cite{gao2018meta,wang2022problem,tan2018hierarchical,gao2019meta}.

In recent years, AI-based methods especially deep reinforcement learning (DRL) techniques become very popular in sequential decision-making problems due to the huge achievement and success in both reinforcement learning (RL) and deep neural network (DNN)\cite{sutton2018reinforcement,mnih2015human,kober2013reinforcement,vinyals2019alphastar}.  
% RL, a method where an agent learns how to act by interacting with the environment and getting feedback from the environment, has shown success in making sequential decisions and has been proved with lots of theoretical convergence properties \cite{sutton2018reinforcement}.
% Tabular RL which stores the value of states in a table requires huge computational resources if the dimension of the state is large. Meanwhile, DNN has been examined successfully to approximate the value of high-dimension input in many subjects such as machine translation and image classification during the last few decades \cite{lecun2015deep}. As a consequence, DRL has been examined with better generalization ability in high dimension states \cite{mnih2015human,kober2013reinforcement,vinyals2019alphastar}.
There is already considerable literature applying DRL techniques to TSC applications and showing its advantages over the classical methods\cite{thorpe1996tra,wen2007stochastic,el2010agent}. 
In early research work, the completely centralized RL technique, where one single agent controls the signal settings of all intersections in traffic grid networks, can learn the optimal policy and achieve a good convergence rate when the size of the traffic network is moderate \cite{thorpe1996tra,wen2007stochastic,el2010agent}.
However, this technique suffers from the scalability issue and is computationally impractical when the size of the traffic network becomes large \cite{haydari2020deep}. 
The completely decentralized technique, where each intersection is assigned to one agent, can overcome the scalability issue by learning the optimal signal control policy either independently or cooperatively \cite{wei2019survey}.
However, MADRL can cause the non-stationarity issue because the transition of the environment is affected by the joint actions of all agents.
The independent RL technique, which maximizes each agent's own reward without considering the change of environment caused by other agents, may fail to converge theoretically\cite{tan1993multi}.
In contrast, the cooperative RL technique aims to maximize not only individual rewards but also local or global rewards. 
Large amounts of work have proposed cooperative RL methods by either applying communication protocols or coordination strategies between agents\cite{van2016coordinated,zhang2022neighborhood,wang2020large}. Nonetheless, the non-stationary issue may still lead to suboptimal control or even convergence failure when the number of agents is large \cite{zhang2021multi}.

To achieve a trade-off between scalability and optimality, some literature has applied a compromised technique which normally involves two stages\cite{chu2016large,tan2019cooperative} where the first stage is to partition a large network into several disjoint small regions and each region is composed of a set of intersections, followed by applying the centralized RL technique to control each region. The global joint action of the whole network is a concatenation of the local action of each region.
% In \cite{chu2016large}, sub-networks are first initialized based on the real-time traffic density between adjacent intersections and further partitioned into small disjoint regions with normalized-cut algorithm \cite{shi2000normalized}  in network theory. Then, approximate Q-Learning is applied to search for sub-optimal actions for all regions. In \cite{tan2019cooperative}, a $ 4 \times 6$ traffic grid is partitioned into four $ 2 \times 3$ disjoint regions and controlled in a decentralized-to-centralized manner by regional DRL (R-DRL) agents and one global coordinator.
However, existing approaches still have the following limitations:
\begin{itemize}
    \item Regions with identical topology may fail to adapt to a different network with a new distribution of intersections' degree and some networks cannot even be partitioned into multiple regions with identical topology. 
    Although the regions in \cite{chu2016large} are partitioned dynamically based on real-time traffic density, the traffic dynamics for regions with the same size but different topologies may differ. Therefore, regions of identical topology lack adaptability, and regions with unrestricted topology may be difficult for RL agents to learn regional traffic dynamics.
    \item Since each region is controlled by one centralized agent, the challenge of optimal joint action searching for the region is naturally inherited from the completely centralized RL technique. Suppose there are $N$ traffic signals and each signal has $k$ phases, then the size of joint action space is $N^k$. As the number of intersections increases, the cardinality of regional joint action space grows exponentially. Hence, the size of the output layer in deep Q-network (DQN) and that in the actor-network of actor-critic architecture also grow exponentially. Besides, existing work bound the size of joint action space by considering traffic signals with two phases but urban traffic signals usually have four phases. 
    % To tackle the above challenge, Tan et al. \cite{tan2019cooperative} applied the deterministic policy gradient (DDPG) and Wolpertinger Architecture (WA) to search for the optimal action for each region. Although DDPG+WA only evaluates some proposed actions mapped from the proto-action, the number of proposed actions is still suggested to be proportional to the cardinality of regional joint action space in practice. 
    Thus, the efficiency of searching for the optimal action for a region needs careful investigation. 
    
    % In \cite{chu2016large}, deep deterministic policy gradient (DDPG) and Wolpertinger Architecture (WA) are applied to search for the optimal action for each region through three steps. Firstly, the actor generates a proto-action in continuous space. Secondly, the $K$-nearest-neighbour (KNN) algorithm is applied to map pro-action to $K$ actions in discrete space. Finally, the critic evaluates the optimal action among $K$ actions. The larger $K$ is, the higher chance the optimal action is in $K$ actions. Therefore, this search method heavily depends on the choice of $K$ and is not efficient in high-dimension action space. 
\end{itemize}

To overcome the above limitations, we propose a constrained partitioning rule and extend Branching Dueling Q-Network (BDQ), whose output layer size grows linearly, with an adaptive computation of target value. More specifically, our main contributions are listed as follows:
\begin{enumerate}
\item We propose a shape-free network partitioning rule which can adapt to networks with different distributions of intersections' degrees. The disjoint region under our partitioning rule includes a central intersection and a subset of its neighboring intersections. Therefore, the maximum distance between any pair of intersections is two inside the region and the topology of the region under the above constraint is a star in graph theory. We further model the network partitioning problem as an optimization problem to minimize the number of regions and give a theoretical analysis of the uniqueness of regions. For those intersections with less than four neighbors, fictitious intersections are introduced to fill the absence during training. Here we define that a region is fully loaded if there is no fictitious intersection in this region. Otherwise, this region is partially loaded.

\item We propose Adaptive BDQ (ABDQ) in which the computation of target value and loss adaptively involves only non-fictitious intersections for different regions in order to mitigate the negative influence of fictitious intersections. 
Since the location of fictitious intersections varies across regions, experimental results show that ADBQ has the potential to control a different number of intersections even in non-grid traffic networks.

\item We evaluate our framework on both real and synthetic datasets. Experimental results show that, although there is no coordination or communication considered among our regional agents, the performance of the proposed approach is better than all baselines. The robustness of our partitioning rule is further examined by employing different region configurations and trying different assignment orders. Also, ABDQ is applied to $2 \times 3$ regions to demonstrate its advantage on partially loaded regions.

\end{enumerate}
 The rest of the paper is arranged as follows: Section \ref{sec: related work} discusses the related work. Section \ref{sec: background} introduces the background and notations of traffic signal control and MARL. Section \ref{sec: regional control formulation} presents our network partitioning rule and the formulation of our regional agents. Section \ref{sec: sec experiments and results} describes the setting of experiments and discusses the results. Section \ref{sec: conclusion} summarises this paper.

\section{Related Work}
\label{sec: related work}
In this section, we review and summarize the related work in DRL-based traffic signal control. 
In recent decades, more researchers have realized that only individual information from a single intersection is not enough to design intelligent signal controllers in large transport systems and have started to utilize local or even global information.
The most straightforward way is to use a single agent to control all signals with the global information of the traffic network \cite{thorpe1996tra,wen2007stochastic,el2010agent}. Although this strategy shows a convergence advantage in small-size traffic networks, it suffers from the scalability issue which leads to a poor convergence rate in large-scale traffic networks \cite{haydari2020deep}. Till now, lots of MADRL algorithms adopting the completely decentralized technique have been proposed to coordinate the action of agents. 
Van der Pol et al. \cite{van2016coordinated} modeled the traffic network as a linear combination of each intersection and applied the max-plus algorithm\cite{kok2005using} to select joint actions. In \cite{tan2018large}, according to the congestion level between the intersection and its neighborhood, the agent can either follow a greedy policy to maximize individual rewards or Neighborhood Approximate Q-Learning to maximize local rewards.

Meanwhile, either explicitly or implicitly information-sharing protocols among completely decentralized agents have been studied in some existing literature.
Arel et al. \cite{arel2010reinforcement} proposed the NeighbourRL in which the observation of one agent is concatenated with the state of the neighboring intersections. 
Varaiya \cite{varaiya2013max} proposed the Max-pressure scheme which considers the difference between the number of waiting vehicles of upstream intersections and that of downstream intersections, and this idea was further applied in PressLight\cite{wei2019presslight}. In CoLight\cite{wei2019colight}, a Graph Attention Network (GAN) was applied to augment the observation of each agent with a dynamically weighted average of its neighboring agents' observations. In contrast, Wang et al. combined GAN with an Actor-critic framework to embed the state of neighboring intersections dynamically\cite{wang2021traffic}. Graph Convolutional Reinforcement Learning\cite{jiang2018graph} can learn underlying relations between neighboring agents and Graph Convolutional Networks (GCN)  have also been applied to automatically capture the geometric feature among neighboring intersections\cite{nishi2018traffic}. Similarly, Devailly et al. proposed inductive graph reinforcement learning where inductive learning and four different types of nodes corresponding to TSC, connection, lane, and vehicle are modeled to support the learning of the surrounding environment and to improve transferability\cite{devailly2021ig}. To speed up the training process, Zang et al. \cite{zang2020metalight} proposed a meta-learning framework named MetaLight to enhance the adaptive ability of RL learning to a new environment by reusing and transferring old experiences. 
Wang et al. \cite{wang2020large} proposed cooperative double Q-learning (Co-DQL) in which the Q-values of agents converge to Nash equilibrium. In their work, the state of one intersection is concatenated with the average value of its neighboring intersections while the reward of one intersection is summed with a weight-average of rewards of its neighboring intersections.
Zhang et al.\cite{zhang2022neighborhood} proposed neighborhood cooperative hysteretic DQN (NC-HDQN) which studies the correlations between neighboring intersections.
In their work, two methods are designed to calculate the correlation degrees of intersection. The first method, named empirical NC-HDQN (ENC-HDQN), assumes that the correlation degrees of two intersections are positively related to the number of vehicles moving between two intersections. In ENC-HDQN, the correlation degrees are always positive and the threshold is manually defined according to the demand of traffic flow. The other method named Pearson NC-HDQN (PNC-HDQN) stores reward trajectories of each intersection and computes Pearson correlation coefficients based on those history trajectories. Unlike ENC-HDQN, PNC-HDQN allows negative correlation degrees between intersections. With correlation degrees of neighboring intersections, the reward of one intersection is summed with its neighboring intersections' rewards weighted by correlation degrees respectively. 
In ENC-HDQN, the weighted sum of rewards can be interpreted as different levels of competition as the coefficients are non-negative. In PNC-HDQN, negative coefficients might be interpreted as cooperation between intersections. In \cite{jiang2021distributed}, the traffic grid network was firstly decomposed into several sub-networks based on the level of connectivity and average residual capacity. Then, all decentralized agents in each sub-network share state and reward.

Apart from the completely decentralized technique, some literature applies the regional control technique.
In \cite{chu2016large}, sub-networks are firstly initialized based on the real-time traffic density between adjacent intersections and further partitioned into small disjoint regions with normalized-cut algorithm \cite{shi2000normalized}  in network theory. Then, approximate Q-Learning is applied to search for sub-optimal actions for all regions. In \cite{tan2019cooperative}, a $ 4 \times 6$ traffic grid is partitioned into four $ 2 \times 3$ disjoint regions and controlled in a decentralized-to-centralized manner by regional DRL (R-DRL) agents and one global coordinator.
In R-DRL, deep deterministic policy gradient (DDPG) and Wolpertinger Architecture (WA) are applied to search for the optimal action for each region through three steps. Firstly, the actor generates a proto-action in continuous space. Secondly, the $K$-nearest-neighbour (KNN) algorithm is applied to map proto-action to $K$ actions in discrete space. Finally, the critic evaluates the optimal action among $K$ actions. The larger the $K$ is, the higher chance the optimal action is in $K$ actions. Therefore, this search method heavily depends on the choice of $K$ and $K$ is suggested to be proportional to the cardinality of regional joint action space in practice. To achieve better cooperation among R-DRL agents, a global coordinator is applied to achieve coordinated R-DRL agents by combining an iterative action search algorithm and one global Q-value estimator. In the iterative action search algorithm, each agent proposes its local optimum actions. Then, from this initial joint action search point, each agent iteratively chooses whether to deviate from its local optimum actions according to the Q-value computed by the global Q-value estimator. Although the iterative search approach offers sufficient trials for different sets of joint actions, the performance of the iterative search approach heavily depends on the convergence of the global function, and the assumption that the global function is well-learned is still too strong for large-scale traffic networks.

\section{Background}
\label{sec: background}

\subsection{Traffic Signal Control Definition}
\label{sec: Traffic Signal Control Definition}
Let us define a traffic network as a directional graph $\mathcal{G}=(\mathcal{V},\mathcal{E})$ where $v \in \mathcal{V}$ represents an intersection and $e_{vu}=(v,u)\in  \mathcal{E}$ represents the adjacency and connection between two intersections. The neighborhood of intersection $v$ is denoted as $NB_v=\{u|(v,u)\in \mathcal{E}\}$ and the degree of one intersection is the size of its neighborhood. $d(v,u)$ denotes the minimum number of edges to connect two intersections $v,u$.

There are two types of approaches for each intersection: The incoming approach is the approach on which vehicles enter the intersection and the outgoing approach is the approach on which vehicles leave the intersection. Each approach consists of a number of lanes and there are incoming lanes and outgoing lanes. The set of entering lanes of intersection $v$ is denoted as $Lane[v]$. A traffic movement is defined as a pair of one incoming lane and one outgoing lane. A phase is a combination of traffic movements that are set to be green. As illustrated right side in Fig. \ref{fig:Traffic network}, one intersection has four phases which are North-South Straight (NS), North-South Left-turn (NSL), East-West Straight (EW), and East-West Left-turn (EWL).
\begin{figure}[t]
    \centering
    \includegraphics[width=0.5\textwidth]{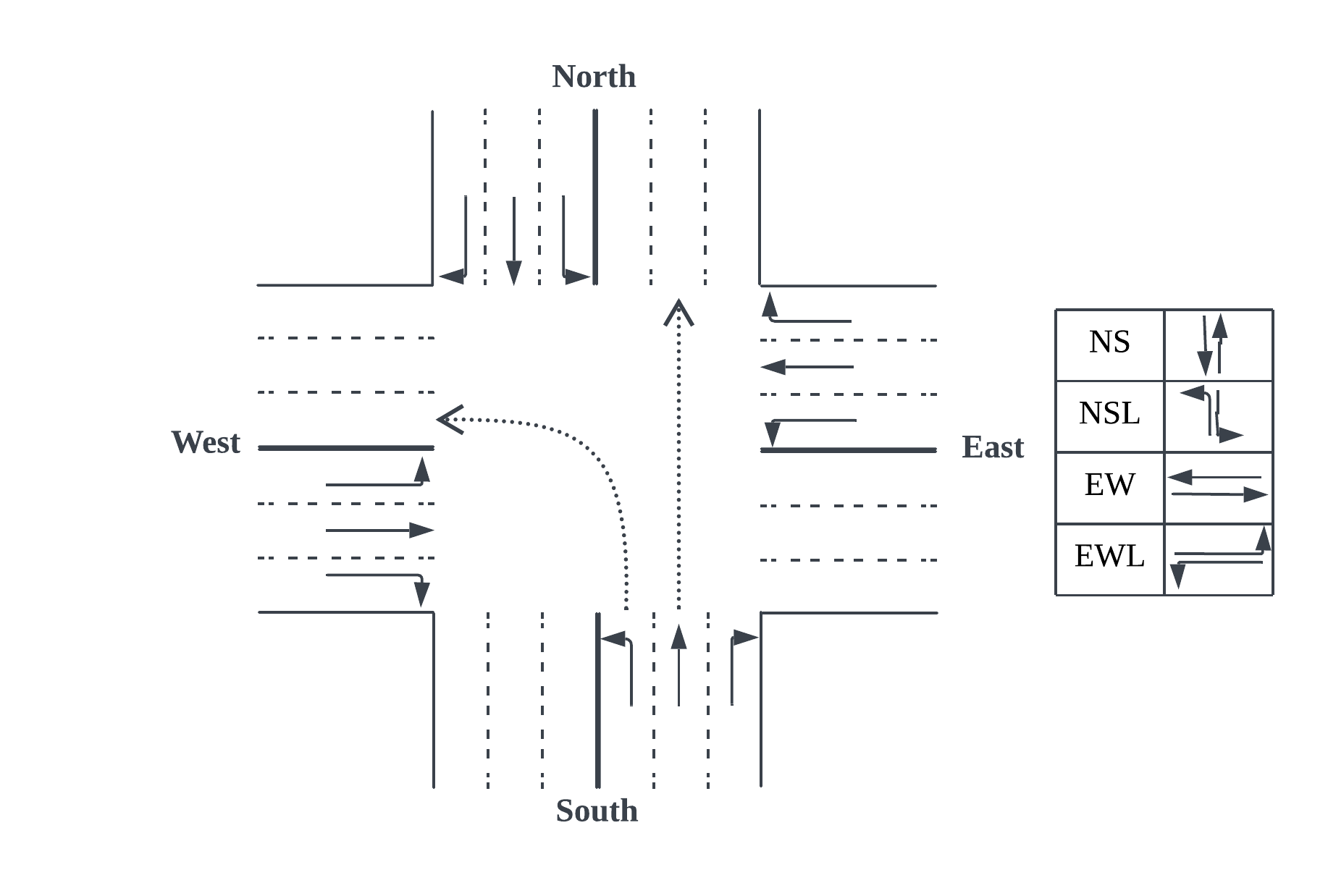}
    \caption{Traffic network example}
    \label{fig:Traffic network}
\end{figure}

\begin{table}[]
    \centering
    \caption{Notion Table}
    \begin{tabular}{|c|c|}
    \hline
        $\mathcal{V}$ & set of all intersections \\
        \hline
        $\mathcal{E}$ & set of all approaches\\
        \hline
        $NB_{v}$ & neighborhood intersections of $v$\\
        \hline
        $d(v,u)$ & minimum hops between $v$ and $u$\\
        \hline
        $Lane[v]$& entering lanes of intersection $v$\\ 
        \hline
        $wait[l]$& number of waiting vehicles on lane $l$\\
        \hline
        $wave[l]$& number of vehicles on lane $l$\\
        \hline
        $phase[v]$ & the phase of intersection $v$\\

         \hline
    \end{tabular}

    \label{tab:Notion Table}
\end{table}
\subsection{Markov Game Framework}
Multi-agent system is usually modelled as a Markov Game (MG) \cite{littman1994markov} which is defined as a tuple$\langle \mathcal{N},\mathcal{S},\mathcal{O},\mathcal{A}, R, P,\gamma \rangle$ where $\mathcal{N}$ is the agent space, $\mathcal{S}$ is the state space, $\mathcal{O}=\{\mathcal{O}_1,...,\mathcal{O}_{|\mathcal{N}|}\}$ is the observation space and $\mathcal{O}_i$ of agent $i$ is observed partially from the state of the system,  $\mathcal{A}=\{\mathcal{A}_1,...,\mathcal{A}_{|\mathcal{N}|}\}$ is the joint action space of all agents, $r_i \in R: \mathcal{O}_i \times \mathcal{A}_1\times \cdots \times \mathcal{A}_{|\mathcal{N|}}\rightarrow \mathbb{R}$ maps an observation-action pair to a real number, $P: \mathcal{S}\times \mathcal{A}_1 \times \cdots \times \mathcal{A}_{|\mathcal{N}|} \times\mathcal{S}\rightarrow [0,1] $ is the transition probability space that assigns a probability to each state-action-state transition and $\gamma$ is the reward discounted factor.

The goal of MG is to find a joint optimal policy $\pi^*=\{\pi^*_1,...,\pi^*_{|\mathcal{N}|}\}$ under which each agent $i$ maximizes its own expected cumulative reward

\begin{equation}
    \mathbb{E}_{\pi_i}[\sum_{k=0}^\infty \gamma^k r_{i,t+k} |o_{i}]
\end{equation}
where $\pi_i$ : $\mathcal{O}_i \times \mathcal{A}_i \rightarrow [0,1]$ maps the observation of agent $i$ to the probability distribution of its action. The action-value (Q-value) of agent $i$ is defined as $ Q_i(o_i,a_i)=\mathbb{E}_{\pi_i}[\sum_{k=0}^\infty \gamma^k r_{i,t+k} |o_{i},a_i]$. 
Tabular Q-learning is a classic algorithm to learn and store action-value \cite{watkins1992q}. 
The update rule is formulated as
\begin{equation}
Q_{i}(o_{i},a_{i})=Q_{i}(o_{i},a_{i})+\alpha(y-Q_{i}(o_{i},a_{i}))
\end{equation}
where $\alpha$ is the learning rate and
\begin{equation}
    y=r_{i}+\gamma \max_{ a_i \in \mathcal{A}_i}  Q_{i}(o_{i}',a_i)
\end{equation}

\subsection{Branching Dueling Q-Network}
\begin{figure*}[t]
    \centering
    \includegraphics[width=\textwidth]{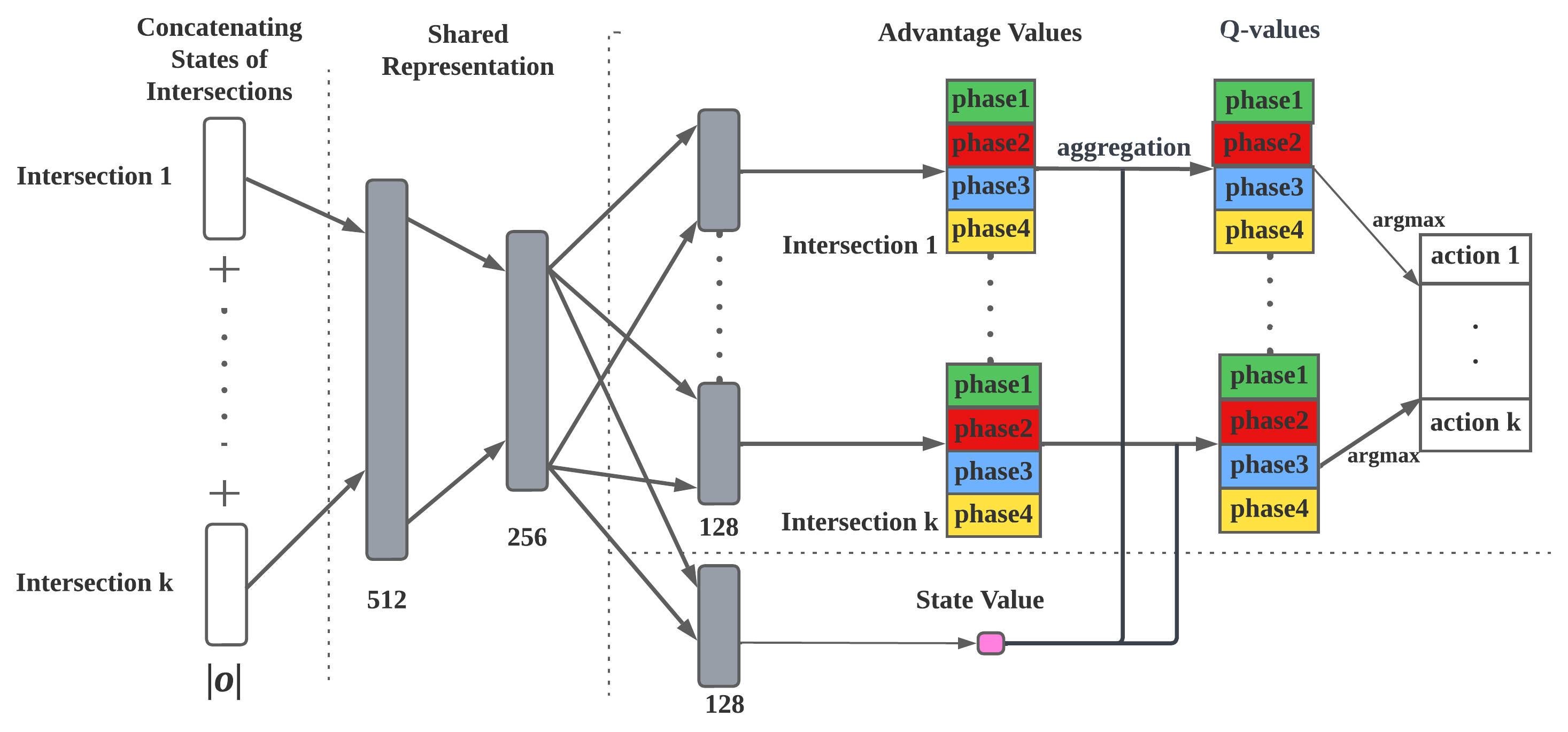}
    \caption{Structure of BDQ. Firstly, the states of intersections are concatenated into a vector and embedded through fully connected layers. Then, the last hidden layer of shared representation is used for both calculations of advantage values and the state value. Next, through the aggregation layer, Q-values of all action branches are calculated with advantage values and the state value based on Eq. (\ref{eq: aggregation }). The size of each layer is annotated below.}
    \label{figStructure of Double-BDQ}
\end{figure*}
In some complex real-life tasks such as robotic control, one task may be divided into several sub-tasks and each sub-task contains a different number of actions. 
Consequently, the size of the action space grows exponentially and the evaluation of each sub-action becomes complex. To improve the efficiency of optimal joint action searching and coordinate sub-tasks to reach a global goal, 
Travakoli et al. proposed a novel agent BDQ \cite{tavakoli2018action} to reduce the output size of the neural network while holding a good convergence property. Suppose that an agent controls $K$ intersections and each intersection has $|\mathcal{A}_k|$ actions, then the cardinality of the action space of this agent is $\prod_{k\in K} |\mathcal{A}_k|$. However, the size of the output of DQN grows exponentially while that of BDQ grows linearly (Fig. \ref{figStructure of Double-BDQ}).

As illustrated in Fig. \ref{figStructure of Double-BDQ}, suppose one RL agent controls $k$ traffic signals and the state dimension $|s_v|$ of each intersection $v$ is the same, then the input layer is a concatenated vector $R^{|\mathcal{O}|\times 1}$ where $|\mathcal{O}|=|s_v|k$. Then the input vector is embedded through two fully connected layers as shared representation layers. The last shared representation layer is first used to compute a common state value and then embedded to get advantage values of each action branch independently. 
Then, the advantage values of each action branch are further aggregated with the state value to calculate the Q-values of each action branch. In \cite{tavakoli2018action}, three ways are proposed to aggregate Q-values and we select the mean advantage aggregation for better stability as the original paper suggested. Formally, suppose that there are $K$ action branches and each action branch has $|\mathcal{A}_k|$ sub-actions, The Q-value of one sub-action is calculated by the common state value and the advantage of this sub-action over the average performance of all sub-actions in this action branch, i.e., 
\begin{equation}
    Q_k(s,a_k)=V(s)+(A_k(s,a_k)-\frac{1}{n} \sum_{ a^{'}_{k} \in \mathcal{A}_k } (A_k(s,a'_k)))
    \label{eq: aggregation }
\end{equation}

To compute the temporal difference target, we choose the mean operator to coordinate all branches to reach a global learning target, i.e., 
\begin{equation}
    y=r+\gamma \frac{1}{N} \sum_k Q^{-}_k(s',\underset{a^{'}_k \in \mathcal{A}_k}{\text{arg max}} Q_k(s',a_k'))
    \label{eq: target value compute}
\end{equation}
Based on Eq. (\ref{eq: aggregation }) and Eq. (\ref{eq: target value compute}), the loss function can be obtained by
\begin{equation}
    L=\mathbb{E}_{(s,a,r,s')\sim D}[\frac{1}{N}\sum_k(y_k-Q_k(s,a_k)^2]
    \label{eq:loss}
\end{equation}

\section{Constrained Network Partition and Formulation of Regional Agent}
\label{sec: regional control formulation}
In this section, we present our partitioning rules for regions and the formulation of regional RL agents. Previous research has demonstrated the benefits of observing either partial or complete states of neighboring intersections \cite{arel2010reinforcement,zhang2022neighborhood,wang2020large}.
% Inspired by these works, we propose our regional control agent with constrained topology and .
Inspired by these work, we assume that an agent can observe the state of all intersections in the corresponding region. We also assume that each intersection has at most four neighbors and this assumption can be trivially extended to networks where the maximum degree is larger than four. 
\subsection{Network Partitioning Rule}
 A region $I_{v}=\{v \cup U\}$ where $U \subseteq NB_{v}$ 
 is a set of intersections including $v$ and a subset of its adjacent intersections. Region configuration $I$ is a union of several regions and follows two constraints :
 \begin{align}
     \label{eq: regionconstrain1}
     \cup_{v}I_{v}&=\mathcal{V}\\
     \label{eq: regionconstrain2}
     I_{v}\cap I_{u}&=\emptyset, \forall I_v,I_u
 \end{align}
These two constraints ensure that all regions are disjoint and each intersection is only assigned to one region.

\subsection{Optimization Problem and Region Construction}
The purpose of regional control is to alleviate the non-stationary issue by restricting the number of agents. Therefore, based on the above two constraints in Eq. (\ref{eq: regionconstrain1}) and (\ref{eq: regionconstrain2}), we further model the network partitioning problem as the dominating set problem \cite{west2001introduction} and construct the region configuration based on the minimum dominating set which is defined as follows.
\begin{definition}[Dominating Set and Domination Number]
\label{def: ds}
    A set $W \subseteq \mathcal{V}$ is a dominating set if every intersection $v \in \mathcal{V}\setminus W $ has a neighbor in $W$. The domination number $\gamma(\mathcal{G})$ is the minimum size of a dominating set in $\mathcal{G}$. The minimum dominating set is a dominating set of size $\gamma(\mathcal{G})$.
\end{definition}
\begin{remark} 
The union of centers of all regions $W=\cup_{I_v}\{v\}$ is a dominating set.
\end{remark}
% \begin{corollary}
% Suppose we have a dominating set $W$, then the minimum distance $d(v,u)$ between all pairs of centers $v,u \in W$ is less than four.
% \end{corollary}
% \begin{proof}
    
% \end{proof}
Therefore, minimizing the number of regions is equivalent to solving the domination number and the minimum dominating set. However, the problems to solve the domination number and find the corresponding dominating set are $NP$-Hard\cite{papadimitriou1998combinatorial}. So we formulate an integer programming problem to solve the optimization problem \cite{duraisamy2021linear}. We assign one binary decision variable $x_v$ to each intersection $v \in \mathcal{V}$. Intersection $v$ is a center if its corresponding variable $x_v=0$. Otherwise, it is a leaf. Formally, the cost function and constraints are:
\begin{align}
    \max \quad & \sum_{v \in \mathcal{V}}x_v \label{eq: obj} \\
\text{s.t. }\quad & \sum_{u\in NB_{v}} x_u +x_v\leq |NB_v|, \quad \forall v \in \mathcal{V}\label{eq: constraint combined}\\
\quad&x_v= 0 \quad \text{or} \quad 1 , \quad \forall v \in \mathcal{V} \label{eq: decision variables} 
\end{align}
The objective to minimize the size of dominating set is equivalent to maximizing the number of leaves. Therefore, Eq. (\ref{eq: obj}) is satisfied. To formulate the constraints to meet Definition \ref{def: ds}, we use binary variables to represent the leaf or center and consider two situations. If one intersection $v$ is a leaf, then at least one of its neighboring intersections is a center. So the sum of decision variables of its neighboring intersections should be less than the number of its neighboring intersections, i.e., 
\begin{equation}
    \sum_{u\in NB_{v}} x_u < |NB_v|, \quad x_v=1 \\
    \label{eq: leaf constrains}
\end{equation}
Since all variables are binary, we can re-write Eq. (\ref{eq: leaf constrains}) into 
\begin{align}
    &\sum_{u\in NB_{v}} x_u \leq |NB_v|-1, \quad x_v=1 \\
    \iff&\sum_{u\in NB_{v}} x_u+x_v \leq |NB_v|, \quad x_v=1 
    \label{eq: leaf constrains2}
\end{align}
Eq. (\ref{eq: leaf constrains2}) holds because the value of $x_v$ is a constant value in this constraint. 
If one intersection $v$ is a center, then there is no constraint on its neighboring intersections.
\begin{align}
\sum_{u\in NB_{v}} x_u+x_v \leq |NB_v|, \quad x_v=0 
    \label{eq: leaf constrains3}
\end{align}
Meanwhile, the Eq. (\ref{eq: leaf constrains3}) holds trivially. As a result, the combination of Eq. (\ref{eq: leaf constrains2}) and (\ref{eq: leaf constrains3}) leads to the constraint condition shown in Eq. (\ref{eq: constraint combined}). Therefore, the integer programming formulation from Eq. (\ref{eq: obj}) to (\ref{eq: decision variables}) is verified.
After solving the optimization problem, we can get the minimum dominating set $W=\{v|x_v=0, v \in \mathcal{V}\}$.
Based on each center, we can construct region configurations around each center (Algorithm \ref{alg:Construct}). 
\begin{algorithm}
\caption{Construction of Region Configuration}\label{alg:Construct}
\begin{algorithmic}[1]
\State Input graph $\mathcal{G}$, minimum dominating set $W$
\State Initialise $key$ to track the assignment of intersection
\State Initialise $I$ to store the configuration of regions
\For {each $v$ in $\mathcal{V}$}
    \If{$v\in W$}
        \State $I[v]\leftarrow\{v\}$ \Comment{initialise region centered at $v$}
        \State $key[v]\leftarrow 1$ \Comment{mark $v$ is assigned}
    \Else
        \State $I[v]\leftarrow \emptyset$ 
        \State $key[v]\leftarrow 0$
    \EndIf
\EndFor
\For {each $v$ in $W$} \Comment{construct regions iteratively}
    \For {each $u$ in $NB_v$}
        \If {$key[u]=0$}
            \State $I[v]\leftarrow I[v]\cup u $ \Comment{$u$ is assigned to $I[v]$} \label{algorithmline: union}
            \State $key[u]\leftarrow 1$ \Comment{mark $u$ is assigned} \label{algorithmline: mark}
        \EndIf
    \EndFor
\EndFor
\State \Return $I$
\end{algorithmic}
\end{algorithm}

In Algorithm \ref{alg:Construct}, line \ref{algorithmline: union} ensures that all intersections are assigned to one region in order to satisfy the first constraint in Eq. (\ref{eq: regionconstrain1}), and line \ref{algorithmline: mark} ensures that each intersection is assigned to exactly one region to satisfy the second constraint in Eq. (\ref{eq: regionconstrain2}). However, Algorithm \ref{alg:Construct} can only construct one region configuration when there exists that a leaf intersection has two alternative neighboring centers. Next, we discuss the uniqueness of the region configuration under Algorithm \ref{alg:Construct} based on one minimum dominating set.  

\begin{theorem}
\label{theorem: uniqueness}
If, $\forall v,u \in W , d(v,u) \geq 3$, then the region configuration $I=\cup_{v\in W} \{I_v\}$ where $I_v=\{NB_v\cup v\}$ is unique.
\end{theorem}

\begin{lemma}
\label{lemma: iff}
$\forall v,u \in W, (NB_v-u) \cap (NB_u-v) =\emptyset$ iff $\forall z \in \mathcal{V}\setminus W$, $|NB_z\cap W|=1$
\end{lemma}
\begin{proof}[Proof for Lemma 1]
We first prove \textit{only if} part of the lemma by contraposition. Assume that $\exists z \in \mathcal{V}\setminus W \text{ such that } |NB_z\cap W|>1$, then this intersection $z$ has at least two neighbors $v,u$ that are centers implying that $z \in NB_v$ and $z \in NB_u$. 
Since $z\neq v$ and $z\neq u$, then $(NB_v-u) \cap (NB_u -v) \neq \emptyset$ which is contradiction. Therefore, it follows that if $\forall v,u \in W, (NB_v-u) \cap (NB_u-v) =\emptyset$, then $\forall z \in \mathcal{V}\setminus W,|NB_z\cap W|=1$.

Next, we prove \textit{if} part of the lemma. Suppose $\forall z \in \mathcal{V}\setminus W, NB_z\cap W=z_d$, then, $\forall z,c \in \mathcal{V}\setminus W $ such that $z_d\neq c_d$, $ |NB_{z_d}\cap NB_{c_d}|\leq 1$. The equality only holds only when $(z_d,x_d)\in \mathcal{E}$. Therefore, $ |(NB_{z_d}-x_d) \cap (NB_{c_d}-z_d)|=0$

Hence, we finish the proof of the lemma.
\end{proof}

\begin{proof}[Proof for Theorem 1]
Since $\forall v,u \in W , d(v,u) \geq 3$, then $NB_v \cap NB_u = \emptyset$ implying $(NB_v-u) \cap (NB_u-v)  =\emptyset$. Then, based on Lemma \ref{lemma: iff}, we know that, $\forall z \in \mathcal{V}\setminus W$, $z$ is connected one unique center in $W$. Then the unique configuration $I$ is $\cup_{v\in W}I_v$ where $I_v=NB_v\cup v$. More generally, $I_v=(NB_v\setminus W)\cup v$
\end{proof}
% \begin{lemma}

% \label{proposition: unique config}
% If, $\forall v,u \in W , d(v,u) \geq 3$, then $ \{NB_v\setminus u\} \cap \{NB_u\setminus v\} =\emptyset$
% \end{lemma}
\begin{remark}
\label{remark: reconstruct}
Let $I_{-v}=I- I_v$ denotes the region configuration except $I_v$. Suppose that $\exists z \in \mathcal{V}\setminus W$ such that $NB_z \cap W=\{v,u\}$ and we have one configuration $I$ where $z \in I_v$. Then we can construct a new configuration $I^\dagger$ by moving $z$ from $I_v$ to $I_u$. Formally, $I^\dagger=I_{-u,v} \cup \{I_v -z\}\cup \{I_u\cup z\}$ is also a valid configuration.
\end{remark}

\subsection{Fictitious Intersection}
\begin{figure}[htb]
    \centering
    \includegraphics[width=0.5\textwidth]{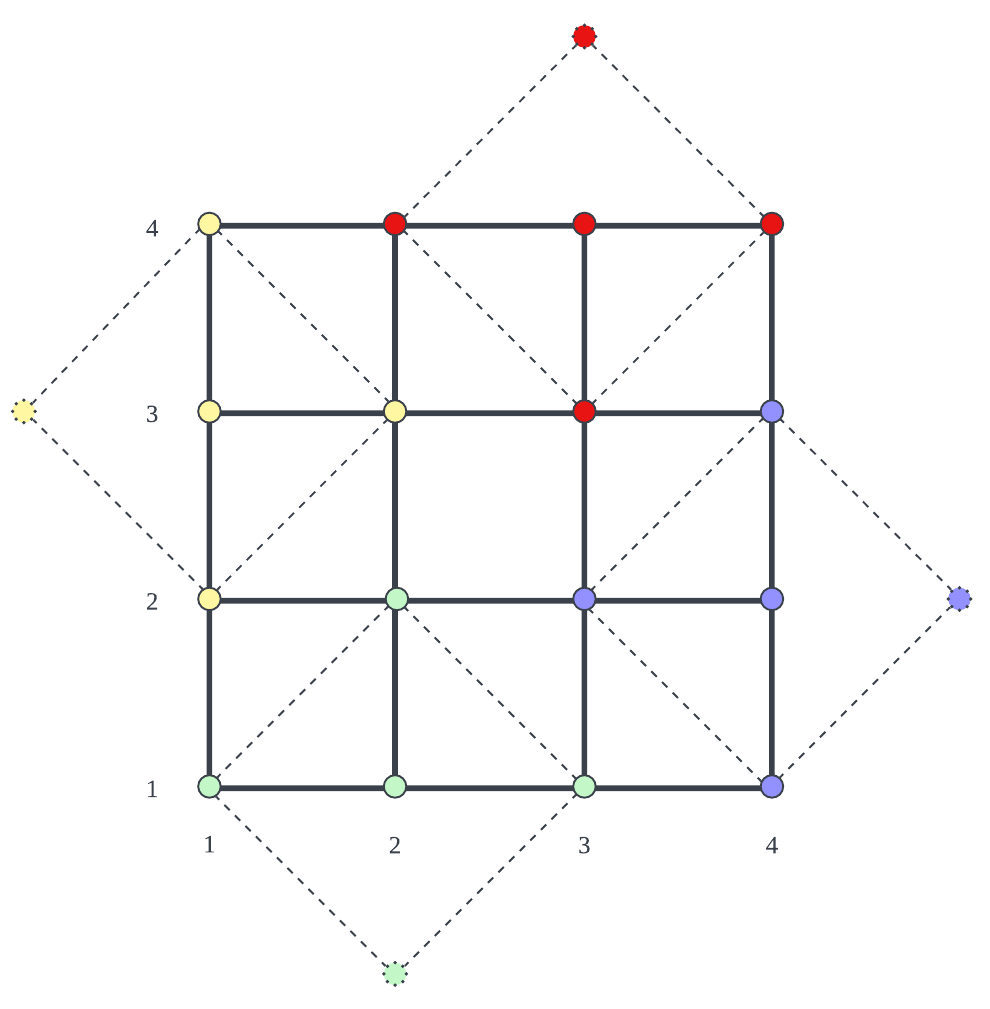}
    \caption{A four-by-four grid with rows and columns indexed and the intersection is labeled with its coordinates. $\gamma(\mathcal{G})=4$ and one corresponding dominating set is $\{1-3,2-1,3-4,4-2\}$. The regions constructed based on this set are $\{I_{1-3}, I_{2-1}, I_{3-4}, I_{4-2}\}$.}
    \label{fig:partition example}
\end{figure}
Our partitioning rule considers the adjacency of intersections and we assume each intersection has at most four adjacent intersections. However, not all regions can contain one center and four leaves. For example, if two centers are adjacent to each other or the center is at the boundary of the grid, then regions constructed based on these centers have less than four leaves. To ensure the completeness of regions, we introduce the fictitious intersection to fill the absence of adjacent intersections and the fictitious intersections will be further handled carefully in the following sections. One region configuration in the $4\times4$ grid traffic network, where centers are at the boundary, is illustrated in Fig. \ref{fig:partition example}.
 
\subsection{RL Formulation of Regional Agents}
In MG, agents interact with environments and learn within multiple episodes. In a complete episode with a length of $\mathcal{T}$ time steps, an agent observes the environment and makes actions at a certain time step $t$. Then the agent receives the reward and the observation of the next time step. 
\subsubsection{Observation Representation}
The observation of an agent $i$ whose region is $I_v$ is a concatenation of all intersections' states in the region. Formally,
\begin{equation}
    o^t_i=\{s^t_u\}_{u\in I_{v}}
\end{equation}
There are lots of types of state representations $s_u$ in literature such as queue length, waiting time and delay\cite{wei2019colight,wei2018intellilight,arel2010reinforcement}. 
In \cite{wei2019presslight}, vehicle wave on each lane is justified with the ability to fully describe the system while the most commonly used state representation is the queue length on each lane. In this paper, we combine the state representation in \cite{chu2019multi,wang2020large} with the signal phase. Formally,
\begin{equation}
s_{u}^t=\{\{wait^t[l]\}_{l \in Lane[u]},\{wave^t[l]\}_{l \in Lane[u]},phase^t[u]\}
\end{equation} 
The state of a fictitious intersection is a vector of zeros.

\subsubsection{Joint Action Space}
As defined in  Fig. \ref{fig:Traffic network}, each intersection has four phases. In TSC, two common settings of action for one intersection are "Switch" or "Choose Phase". In the "Switch" setting, one intersection chooses whether to switch to the next predefined phase or to hold the current phase. Therefore, the phase sequence $P=\{p_1,p_2,..\}$ in this setting follows fixed order and only starting time steps are allowed to deviate. In the "Choose Phase" setting, one intersection chooses which exact phase to run in the next time period. Therefore, both phase sequence and starting time steps vary and this setting offers more flexibility. Moreover, the phase sequence in "Switch" setting is a subset of that in the "Choose Phase" setting.
To both improve the travel efficiency and demonstrate the strength of our model in high-dimensional action space, 
% Our agent controls the signals of all intersections in its region.
 the joint action space of a regional agent $i$ is represented as
\begin{equation}
    \mathcal{A}_i=\{NS,NSL,EW,EWL\}^{|I_v|}
\end{equation}
\subsubsection{Reward Design}
The goal of TSC is to improve traffic conditions in a network such as reducing average travel time. 
However, average travel time can be calculated only after vehicles complete their travel. Therefore, such delayed measurement is not appropriate to be the immediate reward of agents. 
In \cite{zheng2019diagnosing}, for a single intersection, using the queue length as the reward is equivalent to minimizing average travel time. Similar to \cite{tan2019cooperative}, we assume the reward of a region is the summation of the rewards of its intersections. Regional agents learn to minimize the waiting queue length of all intersections in the region, i.e., 
\begin{equation}
    R^t_i=\sum_{u \in I_{v}} r^t_u
\end{equation}
where the reward of a single intersection $u$ at time step $t$ is defined as
\begin{equation}
    r_u^t=-\sum_{l \in Lane[u]}wait^t[l]
\end{equation}
\subsection{Adaptive-BDQ in Regional Signal Control}
In a region of $|I_v|$ intersections, there are $|I_v|$ action branches and each branch corresponds to one particular intersection. Here, we define that, if the corresponding intersection of one action branch is fictitious, then this action branch is idle. Otherwise, this action branch is activated.
In Fig. \ref{fig:partition example}, the fictitious intersections out of the boundary can be modeled as source or sink, and actions on such intersections have no influence on the dynamics inside the region.   

However, the original Equation  \ref{eq: target value compute} for computing target value using an average of all branches may mislead the estimate of the target value and further limit the performance of agents.
Therefore, we propose Adaptive-BDQ (ABDQ) in which the computation of target values in Eq. (\ref{eq: target value compute}) is further modified. For agent $i$, the Q-value of intersection $k$ and action $a_{i,k}$ at time step $t$ is $Q_{i,k}(o^t_i,a_{i,k})$.  Instead of calculating the average of all branches $k$, the Q-value of the next state is the average of activated action branches $\Tilde{k}$, i.e.,
 \begin{equation}
    y_{\Tilde{k}}=r+\gamma \frac{1}{\Tilde{N}} \sum_{\Tilde{k}} Q^{-}_{\Tilde{k}}(o_i^{t+1},\underset{a'_{\Tilde{k}} \in \mathcal{A}_{\Tilde{k}}}{\text{arg max}} Q_{\Tilde{k}}(o_i^{t+1},a_{\Tilde{k}}'))
    \label{eq: target value compute modi}
\end{equation}
where $\Tilde{N}$ is the number of activated action branches. In the loss function, only errors of activated branches are involved:
\begin{equation}
    L=\mathbb{E}_{(o,a,r,o')\sim D}[\frac{1}{\Tilde{N}}\sum_{\Tilde{k}}(y_{\Tilde{k}}-Q_{\Tilde{k}}(o,a_{\Tilde{k}})^2]
    \label{eq:loss modi}
\end{equation}

\subsection{Components and Pipeline of Training Framework}
In RL, there are two major components-- Environment and agent. Agent receives observations from the environment and returns actions. The environment then moves to the next step and passes transition tuples and the next observation to agents. The architecture of our training is illustrated in  Fig. \ref{fig:training framework}. 
Since the state of the simulator needs further augmented into observation, a Pipeline class is introduced to process data from the simulator and agents. The procedure of Pipeline is listed in Algorithm \ref{alg:pipeline}. As shown in Fig. \ref{fig:partition example}, the position of fictitious intersection varies in different regions. Thus, the indices of idle action branches are different among different regions. To accelerate convergence and improve generalized ability, we adopt the centralized learning but decentralized execution (CLDE) paradigm where agents share network parameters and experience memory.
\begin{figure}[t]
    \centering
    \includegraphics[width=0.5\textwidth]{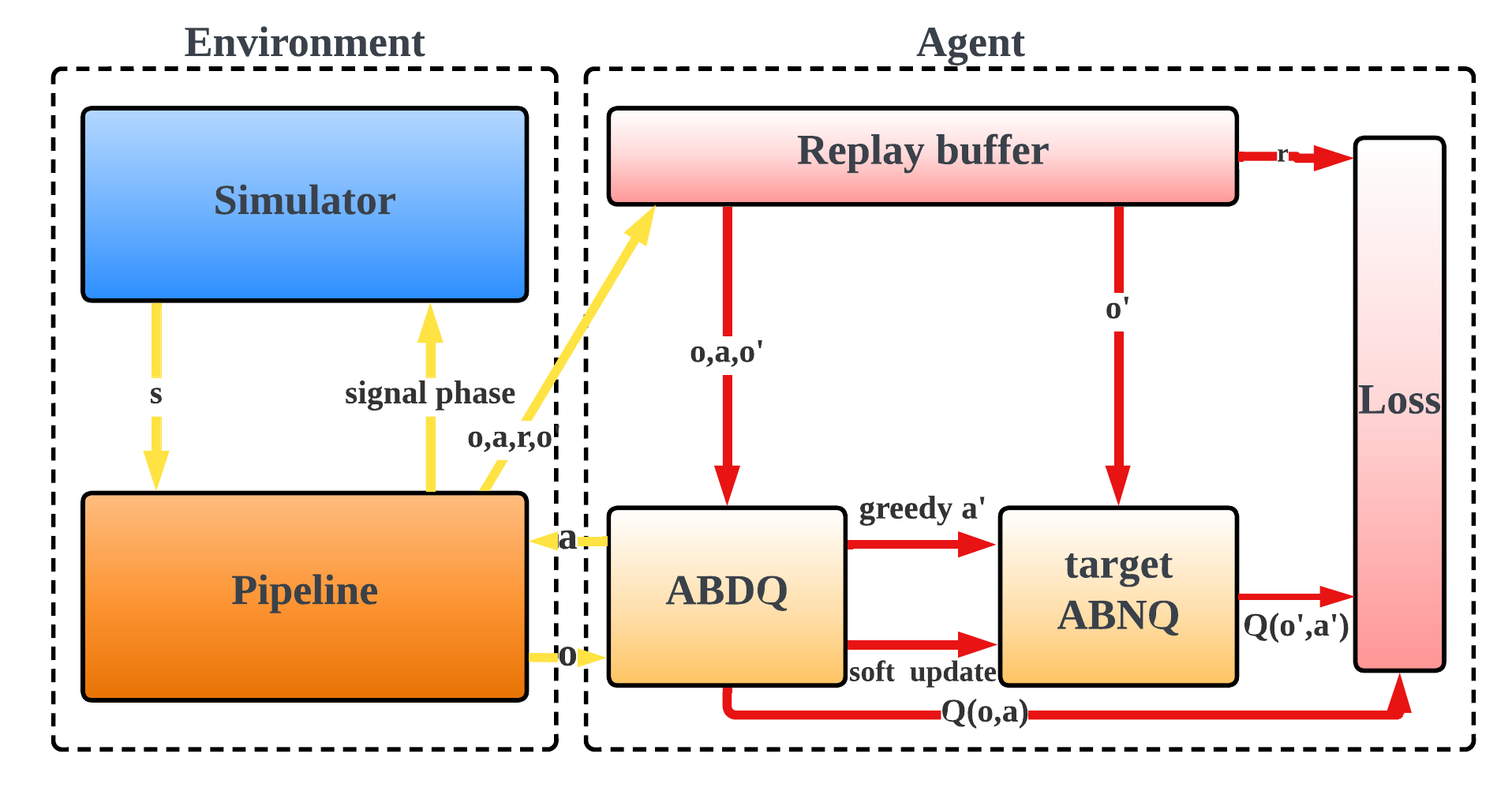}
    \caption{The training framework of RegionLight. Two major components communicate and exchange data through Pipeline Class. The working flow is separated into two parts. The first part is the interaction between the agent and the environment (Yellow arrows) and the second part is the learning and network updating of the agent (Red arrows). In the simulation step of the interaction part, the simulator passes the state to the pipeline and the pipeline then generates observations for agents. Based on the observation, the agent chooses joint action $a$ under $\epsilon$-greedy policy and passes the joint action to the pipeline. Finally, the pipeline passes signal phases to simulation and moves to the next simulation step.}
    \label{fig:training framework}
\end{figure}
\begin{algorithm}
\caption{Algorithm for Pipeline}\label{alg:pipeline}
\begin{algorithmic}[1]
\State Initialise ADBQ $\theta_i$ and target ADBQ $\theta^-_i\leftarrow\theta_i$
\State Initialise Replay Memory $D_i$, Region Configuration $I$
\While{$i < episode$}
    \State $s\leftarrow$ environment reset
    \State $o_i \leftarrow$ generate observation based on $s,I$
    \While{$t<T$}
        \If{$rand<\epsilon$}
            \State $a_i\leftarrow$ random joint action
        \Else
            \State $a_i\leftarrow \cup_{\Tilde{d}} \underset{a^{'}_{\Tilde{d}} \in A_{\Tilde{d}}}{\text{arg max}}Q_{\Tilde{d}}(o_i,a_{\Tilde{d}}')$ 
        \EndIf
        \State $r,s'\leftarrow$ env step after all agents choose actions
        \State calculate$R_i$ and generate $o_i'$ based on $s',I$
        \State store transition$(o_i,a,R_i,o_i')$ to $M_i$
        \State Update $\theta$ by Eq. (\ref{eq:loss}) with target value by Eq. (\ref{eq: target value compute modi})
        \State $\theta_i^-\leftarrow (1-\tau)\theta_i^-+\tau \theta_i$ for certain step
        \State $o_i\leftarrow o_i'$
        
    \EndWhile

\EndWhile
\end{algorithmic}
\end{algorithm}
\section{Experiments and Results}
\label{sec: sec experiments and results}
In this section, we test our method in both real and synthetic grids and compare it with other novel MADRL frameworks. To show the robustness of our region design, 
two different minimum dominating sets are used to construct region configurations in two $4 \times 4$ grid networks, and different shuffled assignment orders are applied in the $16 \times 3$ grid network.
To evaluate the improvement of ABDQ in controlling partially loaded regions, we test ABDQ on minimum-dominating-set-based regions which contain one fictitious intersection, and $2\times 3$ regions which contain two fictitious intersections.

\subsection{Experiment Scenario}
The traffic simulator CityFlow we used in this paper is an open-source simulator \cite{zhang2019cityflow}. In our experiment, one real $4\times 4$ grid (Hangzhou), one synthetic $4\times 4$ grid and one real $16 \times 3$ grid (Manhattan) are used. Roadnets of Hangzhou and Manhattan are illustrated in Fig. \ref{fig:roadnet}.
\begin{figure}[htb]
    \centering
    \subfloat[Gudang, Hangzhou]{
    \includegraphics[width=0.24\textwidth]{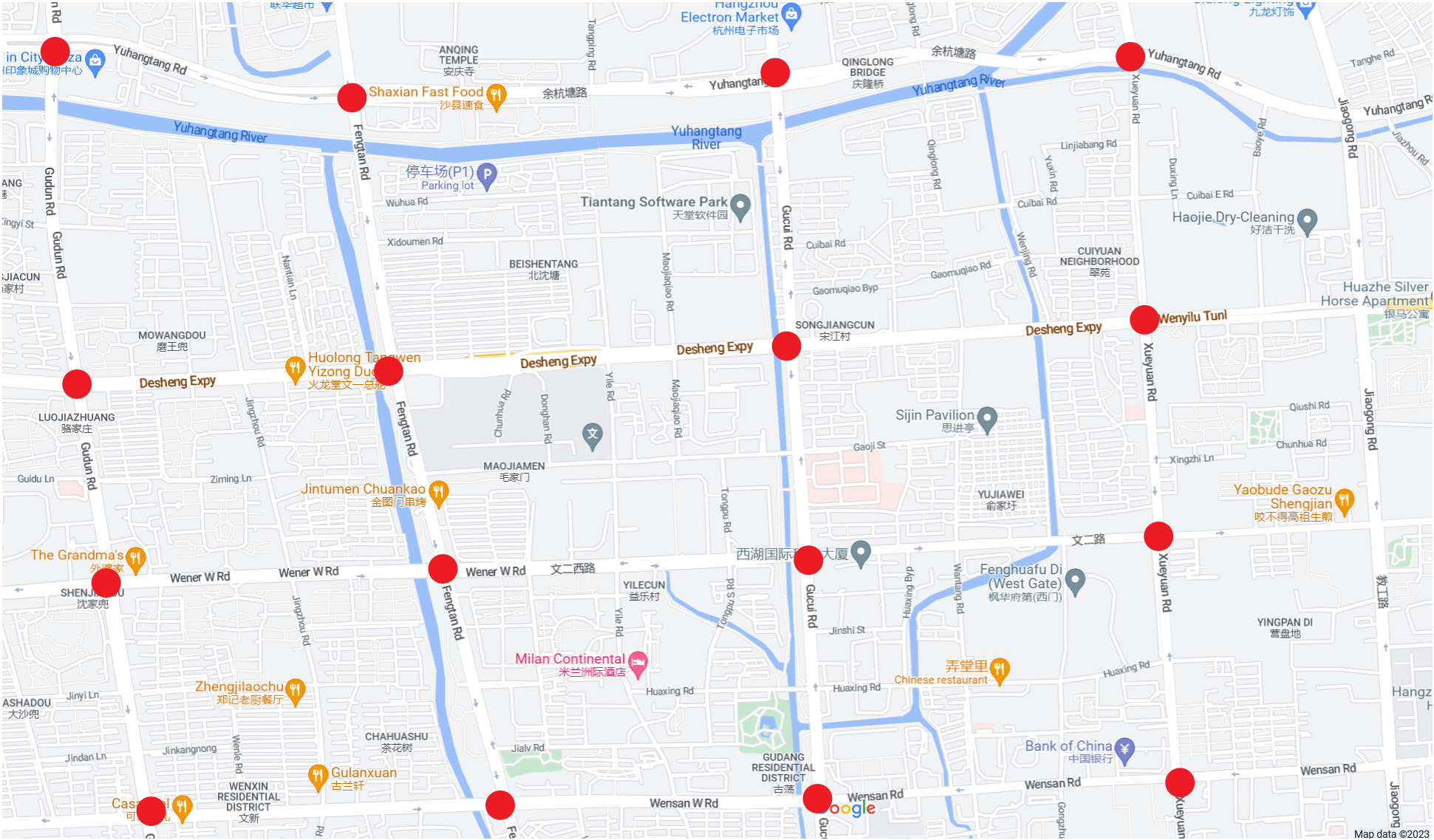}
}
    \hfill
    \subfloat[Manhattan]{
    \includegraphics[width=0.24\textwidth]{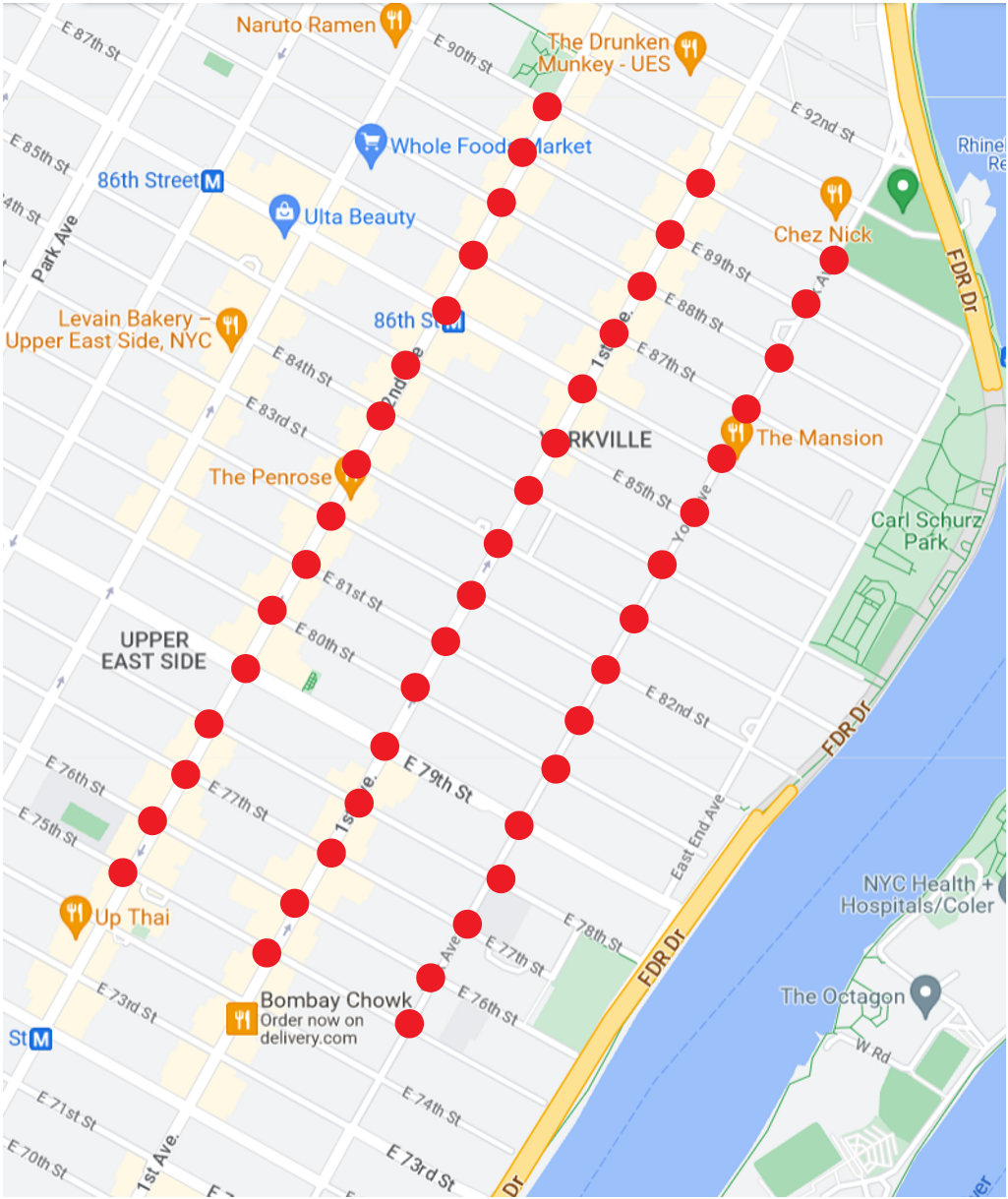}
    \hfill
    }
    \caption{Road networks for Hangzhou and Manhattan Grid. Traffic signals controlled by agents are marked with red dots.}
    \label{fig:roadnet}
\end{figure}
% The distance between two adjacent intersections in the Hangzhou network is 600 meters and that in the synthetic network is 300 meters.
% In the Hangzhou network, two traffic flows, \textcolor{red}{whose volumes are derived from camera data}, based on camera data are used, i.e., one is in flat hours and the other is in peak hours. \textcolor{red}{The turning ratio for both flows is synthesized from the statistics of taxi GPS data and is distributed as 10\% (left turn), 60\% (straight), and 30\% (right turn).}
% % The data is based on the camera data in Hangzhou and is further simplified. 
% In the synthetic network, the flow is generated according to Gaussian distribution, and the turning ratio for synthetic flow is \sout{distributed as 10\% (left turn), 60\% (straight), and 30\% (right turn)} \textcolor{red}{distributed the same as Hangzhou flow}.
% In the Manhattan network, the flow is sampled from taxi trajectory data.
In the Hangzhou network, two traffic flows, whose volumes are derived from camera data, are used, i.e., one is in flat hours and the other is in peak hours. The turning ratio for both flat and peak flows is synthesized from the statistics of taxi GPS data.
In the synthetic network, the flow is generated according to Gaussian distribution, and the turning ratio for synthetic flow is distributed as 10\% (left turn), 60\% (straight), and 30\% (right turn).
In the Manhattan network, the flow is sampled from taxi trajectory data and the turning ratios of each movement at different intersections are not identical.
The statistics of flows and the distance between neighboring intersections are listed in Table \ref{tab:Flow statistics}. All these datasets are open-sourced\footnote{https://traffic-signal-control.github.io/\#open-datasets}.

\begin{table}[htb]
    \centering
        \caption{Flow and Network statistics}
    \begin{tabular}{|c|c|c|c|}
        \hline
        \multirow{2}*{Scenario} & 
        \multicolumn{2}{c|}{Arrival Rate(vehicles/s)}&
        \multirow{2}*{Distance} \\ \cline{2-3} 
        \multicolumn{1}{|c|}{} & 
        \multicolumn{1}{|c|}{Mean}&Std &  \\ \hline
        (Hangzhou) Flat & 0.83 & 1.33 & 600m  \\ \hline
        (Hangzhou) Peak & 1.82 & 2.15 &600m \\ \hline
        Synthetic-1 & 3.12 & 4.08 &300m  \\ \hline
        Synthetic-2 & 3.12 & 4.08 &600m  \\ \hline
        Manhattan & 0.78 &2.49 &100m (NS), 350m (EW) \\ \hline
    \end{tabular}
    \label{tab:Flow statistics}
\end{table}
The length of each episode is set to 4000s. To avoid signals flicking too frequently, all agents perform actions for every $\Delta t=$10s and no yellow phase is inserted between different phases. Then, the length $\mathcal{T}$ of one episode is 400. 

We compare our agent with both non-RL and RL baselines. Non-RL baselines include Fixed time, Max Pressure, SOTL, Genetic Algorithm and Artificial Bee Colony. RL baselines include NeighbourRL, R-DRL, CoLight\footnote{https://github.com/wingsweihua/colight}, PNC-HDQN and ENC-HDQN\footnote{https://github.com/RL-DLMU/PNC-HDQN}. For ABDQ, the neural network structure follows the BDQ \cite{tavakoli2018action} and is illustrated in Fig. \ref{figStructure of Double-BDQ}. The size of hidden layers of NeighbourRL is the same as ABDQ. For R-DRL, the structures of actor and critic follow the description in \cite{tan2019cooperative}. Although a global critic is proposed to coordinate R-DRL, the convergence of this global critic is not guaranteed. So only R-DRL agents are compared. The hyperparameters for our agent, NeighbourRL and R-DRL are listed in Table \ref{tab:Hyperparameter}. 

For other baselines, we run the source code for fairness.The proposed framework is open-sourced\footnote{https://github.com/HankangGu/RegionLight} and is coded in Python. Gurobi, a third-party Python package, was used to formulate and solve the linear integer programming problem\cite{gurobi}. The RL agent was coded with Tensorflow developed by Google\cite{abadi2016tensorflow}.

\subsection{Metric}
\begin{table*}[ht]
    \centering
        \caption{Numerical Statistics for non-RL Baselines}
   \begin{tabular}{|c|c|c|c|c|c|c|c|c|}
    \hline
          \multirow{2}{*}{Flow}& \multirow{2}{*}{Metric}&\multirow{2}{*}{Fixed}&\multirow{2}{*}{Max Pressure}&\multirow{2}{*}{SOTL}&Genetic&Artificial&RegionLight(OURS)&RegionLight(OURS) \\
                   &&&&&Algorithm&Bee Colony&$\gamma=0.9$&$\gamma=0.99$\\
         \hline
         \multirow{3}{*}{Flat}
         &ATT&482.19&434.65&364.42&399.91&345.10&319.14$\pm$0.43&319.28$\pm$0.43 \\
         &AQL&0.57&0.52&0.25&0.34&0.16&0.07$\pm$0.0016&0.07$\pm$0.0017\\
        &TP&2810&2854&2919&2923&2942&2963.27$\pm$0.89&2963.75$\pm$0.88 \\

         \hline
         \multirow{3}{*}{Peak}&
         ATT&803.78&525&435.72&574.2&567.87&402.02$\pm$3.00&402.93$\pm$3.27\\
        &AQL&1.8&1.42&0.82&1.45&0.95&0.44$\pm$0.01&0.45$\pm$0.01\\
        &TP&5105&6176&6224&6046&5926&6382.19$\pm$10.9&6385.8$\pm$11.26\\

         \hline
         \multirow{3}{*}{Synthetic-1}&
         ATT&548.77&235.29&284.47&-&531.55&206.6$\pm$1.3&207.60$\pm$1.2\\
        &AQL&3.32&1.19&1.78&-&3.22&0.85$\pm$0.016&0.86$\pm$0.013\\
        &TP&9553&11181&11166&-&9648&11227.93$\pm$1.07&11227.056$\pm$0.94\\
         \hline
         \multirow{3}{*}{Synthetic-2}&
         ATT&685.63&474.64&454.74&557.32&519.01&381.38$\pm$ 1.20&382.82$\pm$1.26\\
        &AQL&4.65&2.15&1.76&3.02&2.81&0.85$\pm$ 0.015&0.87$\pm$0.016\\
        &TP&9218&10690&10810&10094&10146&11033.78$\pm$8.75&11029.40$\pm$10.76\\
         \hline
        \multirow{3}{*}{Manhattan}&
         ATT&1198.24&287.62&340.67&-&-&176.46$\pm$0.316&177.75$\pm$0.415\\
        &AQL&1.09&0.13&0.21&-&-&0.024$\pm$0.0004&0.025$\pm$0.0004\\
        &TP&1116&2799&2754&-&-&2824$\pm$0.0&2824$\pm$0.0\\

         \hline
    \end{tabular}
    \label{tab:performance evaluation non_rl}
\end{table*}
\begin{table}[htb]
    \centering
    \caption{Hyperparameter Summary}
    \begin{tabular}{|c|c|c|}
     \hline
        Component&Hyperparameter &Value  \\
        \hline
         \multirow{7}*{ABDQ}& $\gamma$ &0.9 \& 0.99 \\ \cline{2-3} &
                                Learning rate $\alpha$ &0.0001\\\cline{2-3}&
                                Replay Buffer Size&200000\\\cline{2-3}&
                                Network optimizer& Adam\\\cline{2-3}&
                                Activation Function &Relu\\\cline{2-3}&
                                $\tau$ &0.001\\\cline{2-3}&
                                Batch Size &32 \\
                                     \hline
         \multirow{3}*{R-DRL}&$k$ & 128 (4$\times$ 4), 1024 (16$\times$ 3)\\\cline{2-3}&
                              $\alpha_{\text{critic}}$&0.0001\\\cline{2-3}&
                               $\alpha_{\text{actor}}$&0.00001\\    
         \hline
         \multirow{3}*{$\epsilon$-greedy Policy}& $\epsilon_{max}$ &1 \\ \cline{2-3}&
                                                  $\epsilon_{min}$ &0.001\\\cline{2-3}&
                                                  decay steps & 20000\\
                          
         \hline
    \end{tabular}
    \label{tab:Hyperparameter}
\end{table}
Similar to \cite{zhang2022neighborhood}, we choose three metrics to evaluate the performance of agents.
\begin{itemize}
    \item Average Travel Time (ATT): average of all vehicles' travel time. Since the time step of our simulation is larger than the arrival time span, all vehicles can travel in the network for enough time and the computation of ATT is more complete and less affected by vehicles which just depart;
    \item Average Queue Length (AQL): average queue length on each lane of all intersections. The definition of the reward of our agent is the sum of queue length. So AQL is a direct numerical interpretation of reward;
    \item Throughput (TP): number of vehicles that arrive at the destination. While reducing ATT, we also want to increase TP so that benefits are maximized.
\end{itemize}
\subsection{Overall Performance}
\begin{figure*}[htb]
    \centering
    \subfloat[(Hangzhou)Flat]{
    \includegraphics[width=0.23\textwidth]{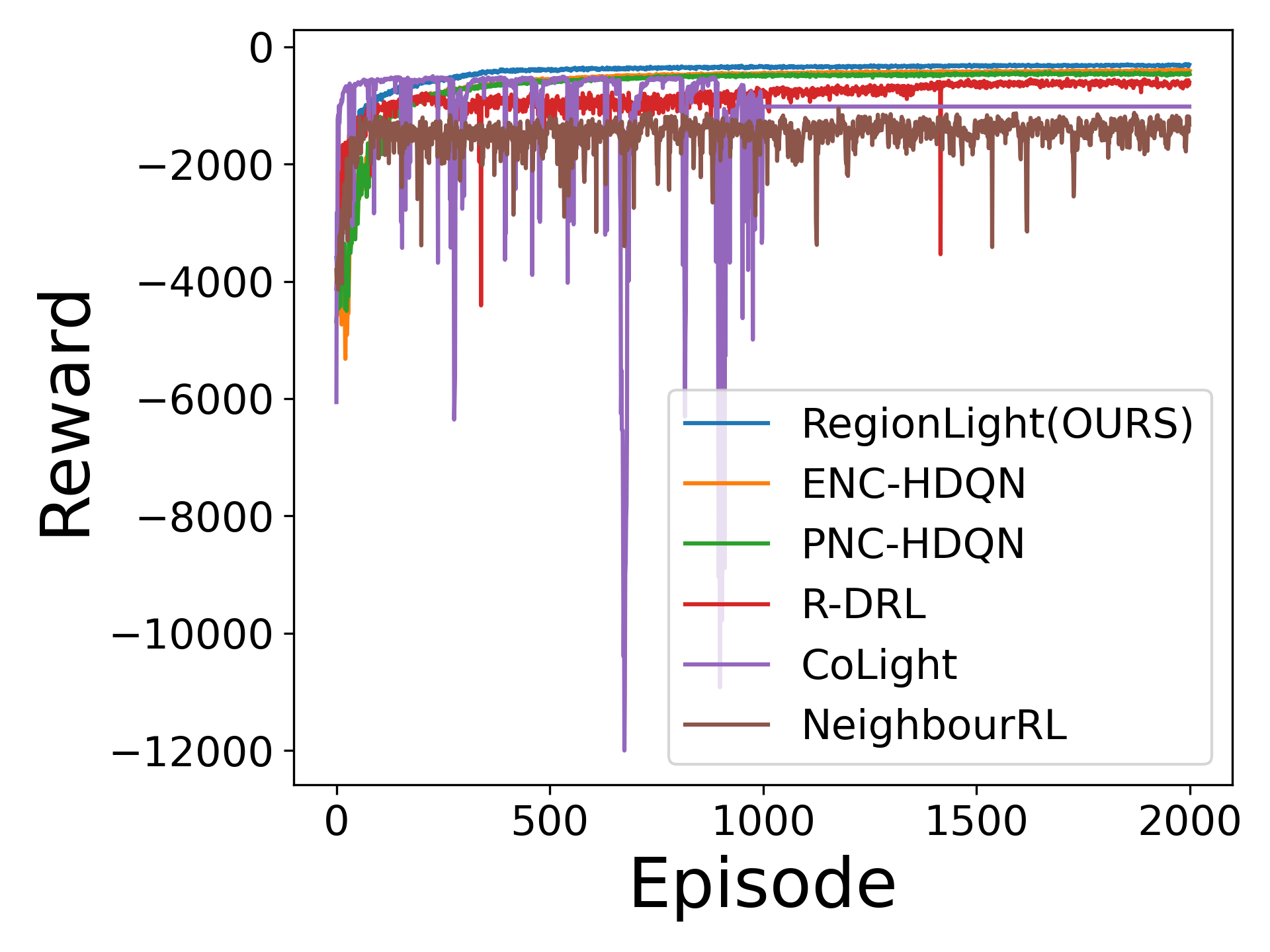}
}
    \hfill
    \subfloat[(Hangzhou)Peak]{
    \includegraphics[width=0.23\textwidth]{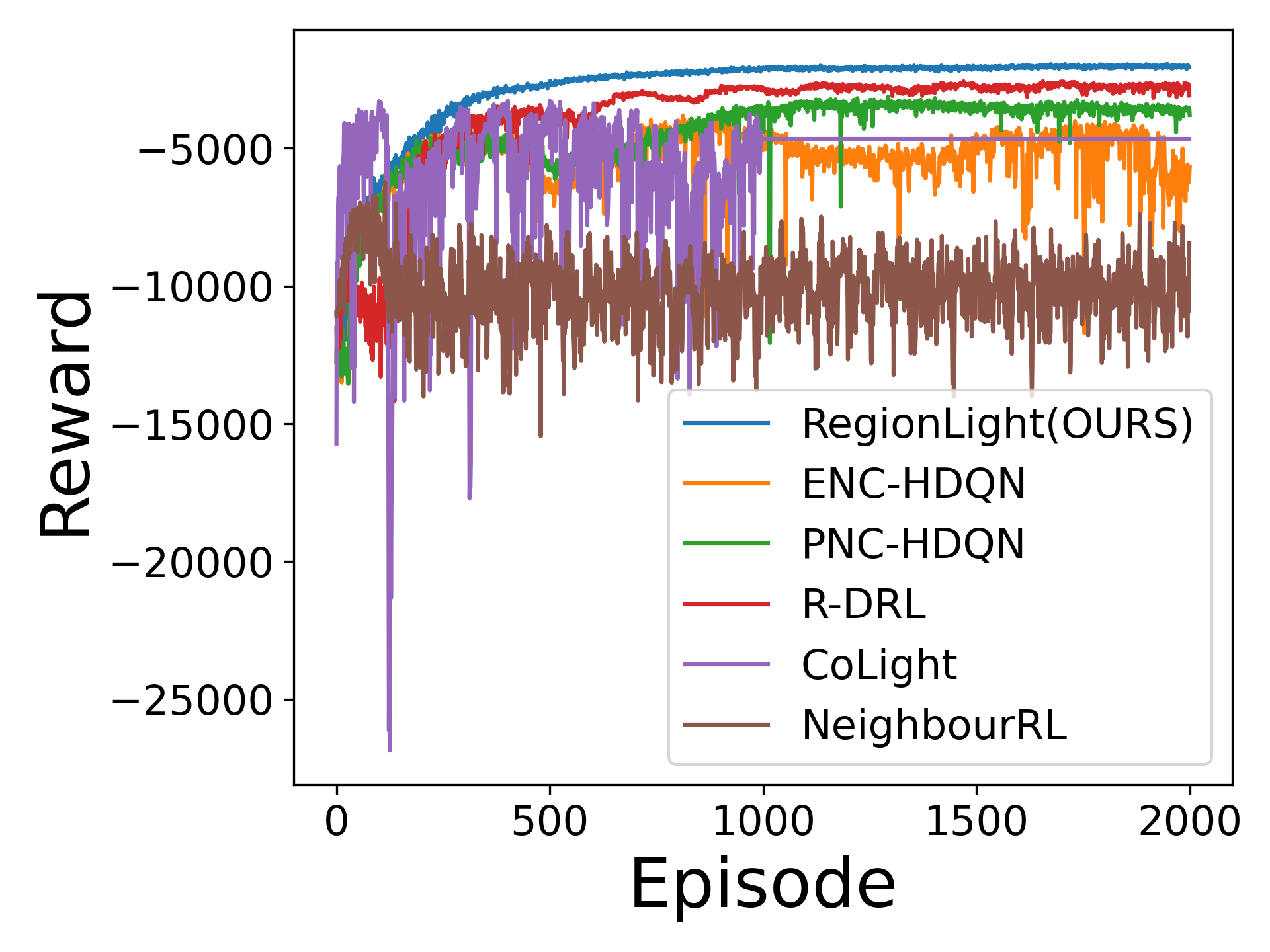}
    }
    \hfill
      \subfloat[Synthetic-1]{
    \includegraphics[width=0.23\textwidth]{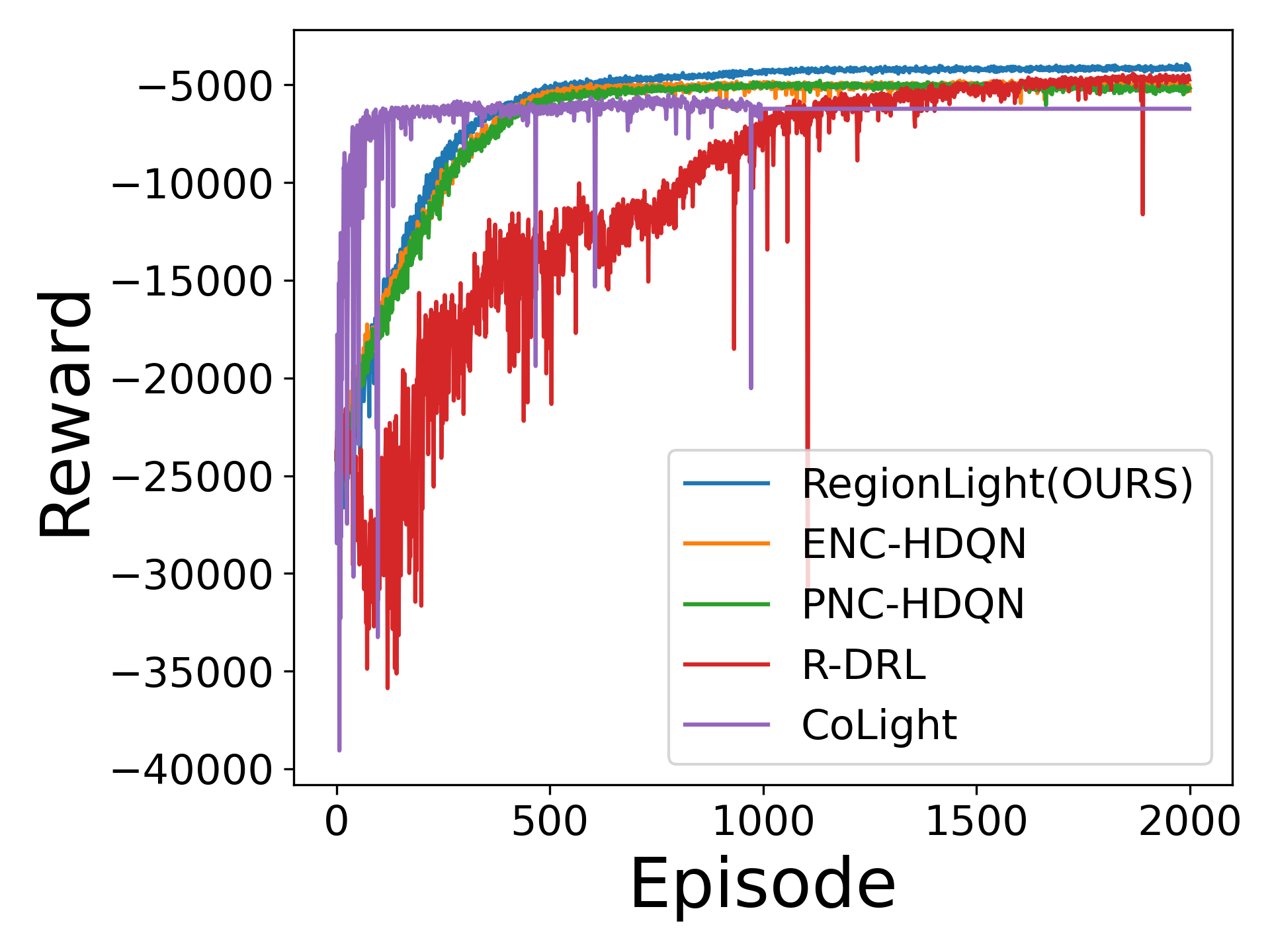}
    }
        \hfill
      \subfloat[Synthetic-2]{
    \includegraphics[width=0.23\textwidth]{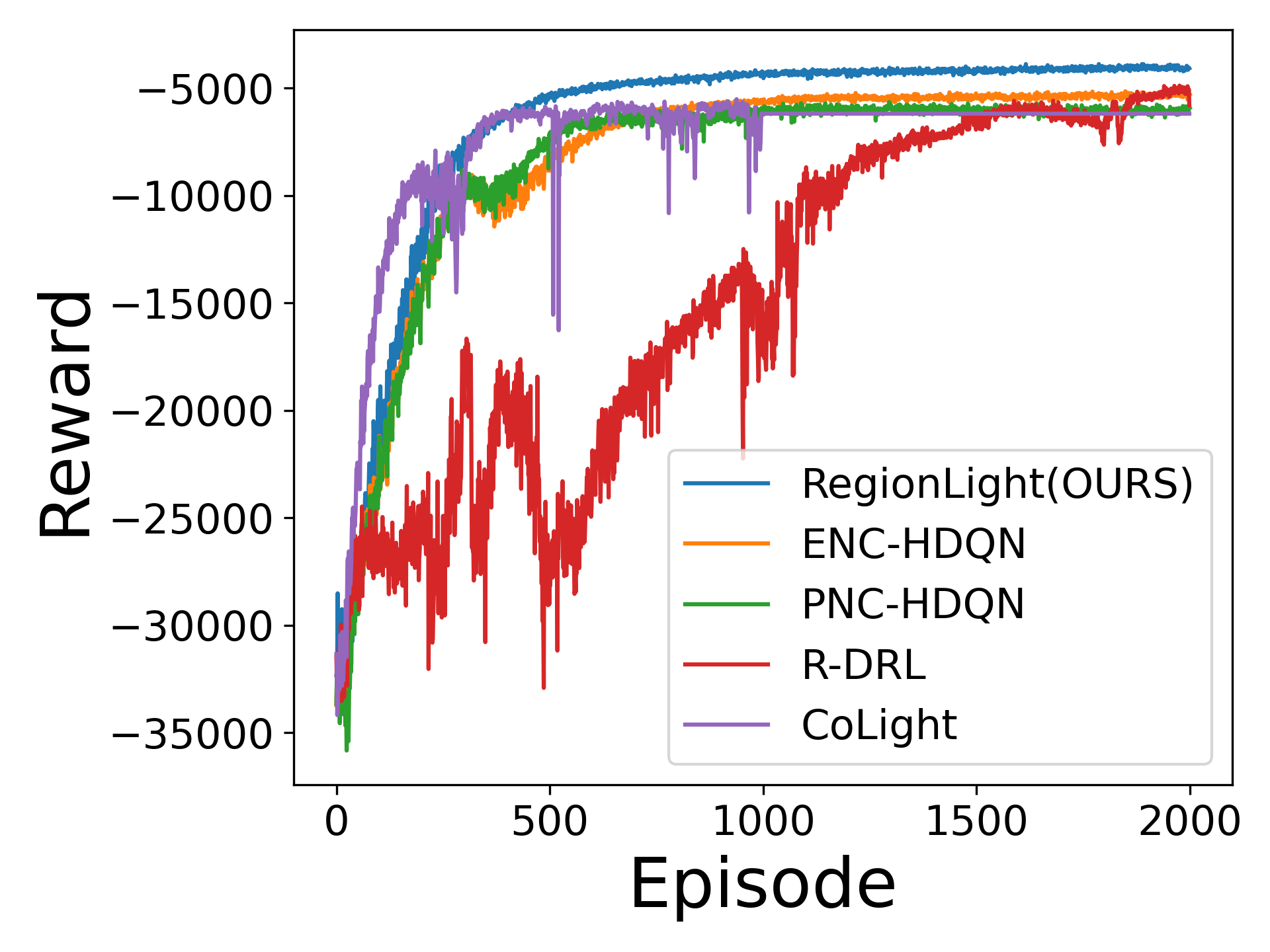}
    }
        \hfill

    \subfloat[Flat after Convergence]{
    \includegraphics[width=0.23\textwidth]{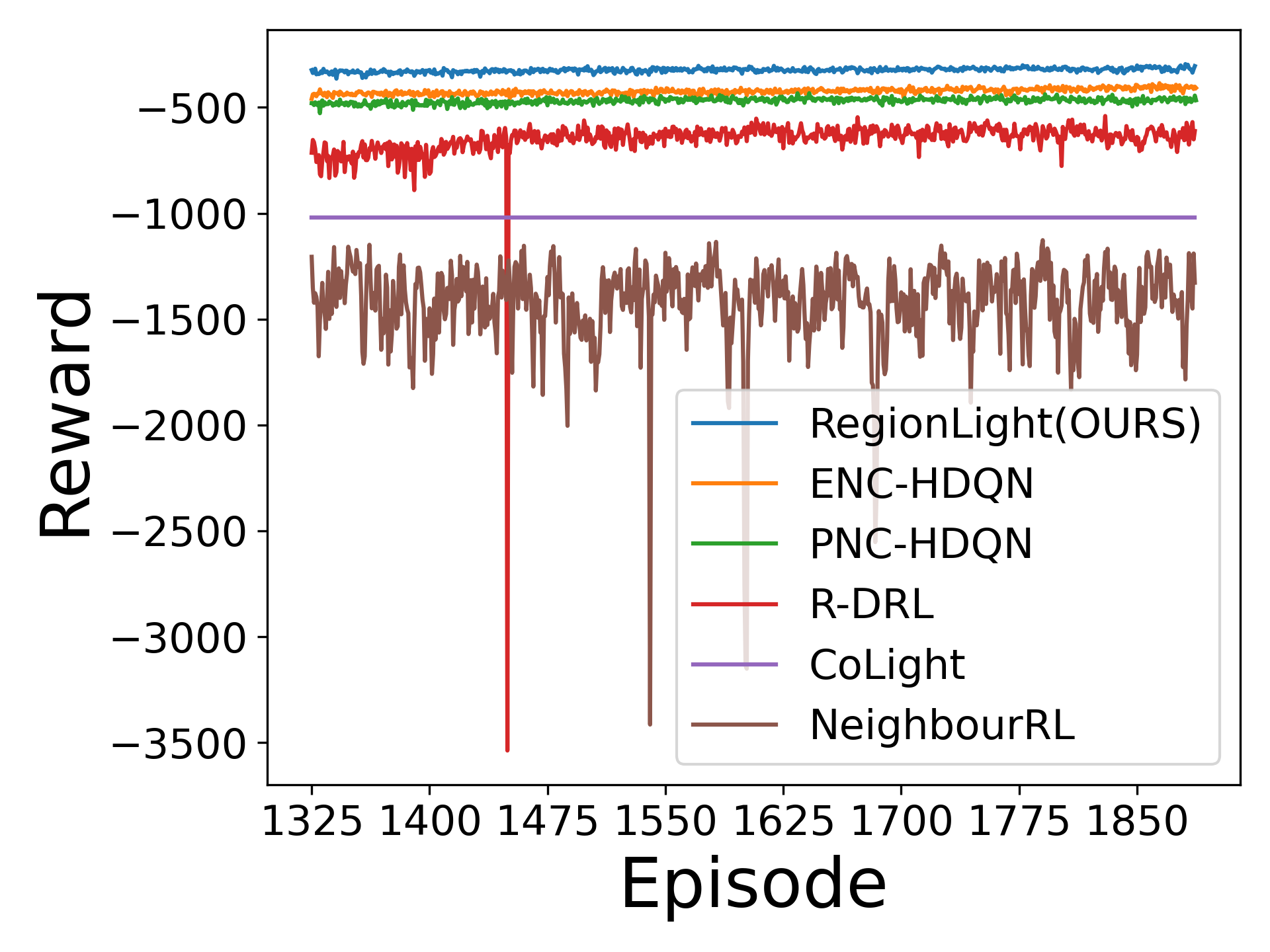}
}
    \hfill
    \subfloat[Peak after Convergence]{
    \includegraphics[width=0.23\textwidth]{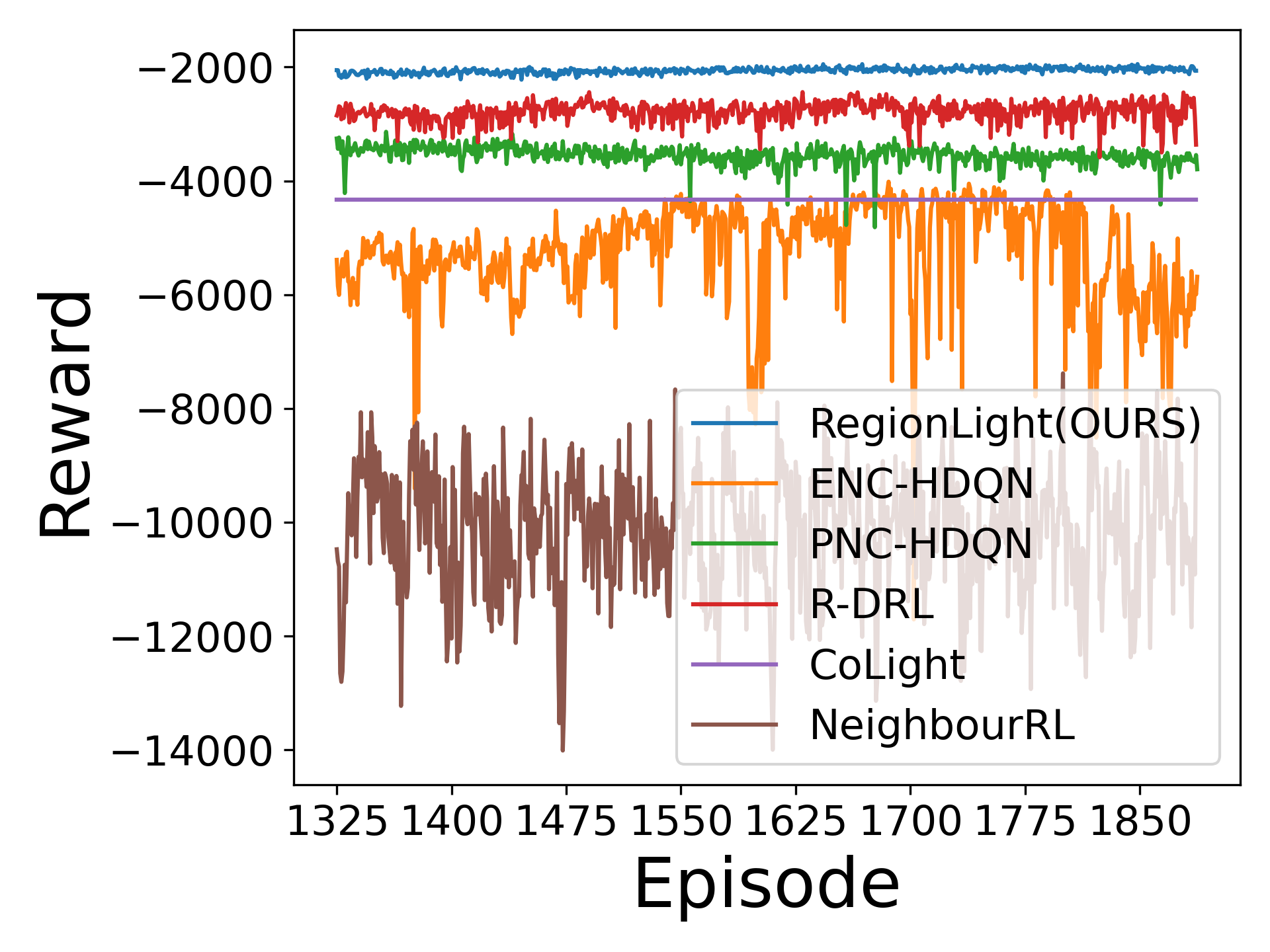}
    }
    \hfill
      \subfloat[Synthetic-1 after Convergence]{
    \includegraphics[width=0.23\textwidth]{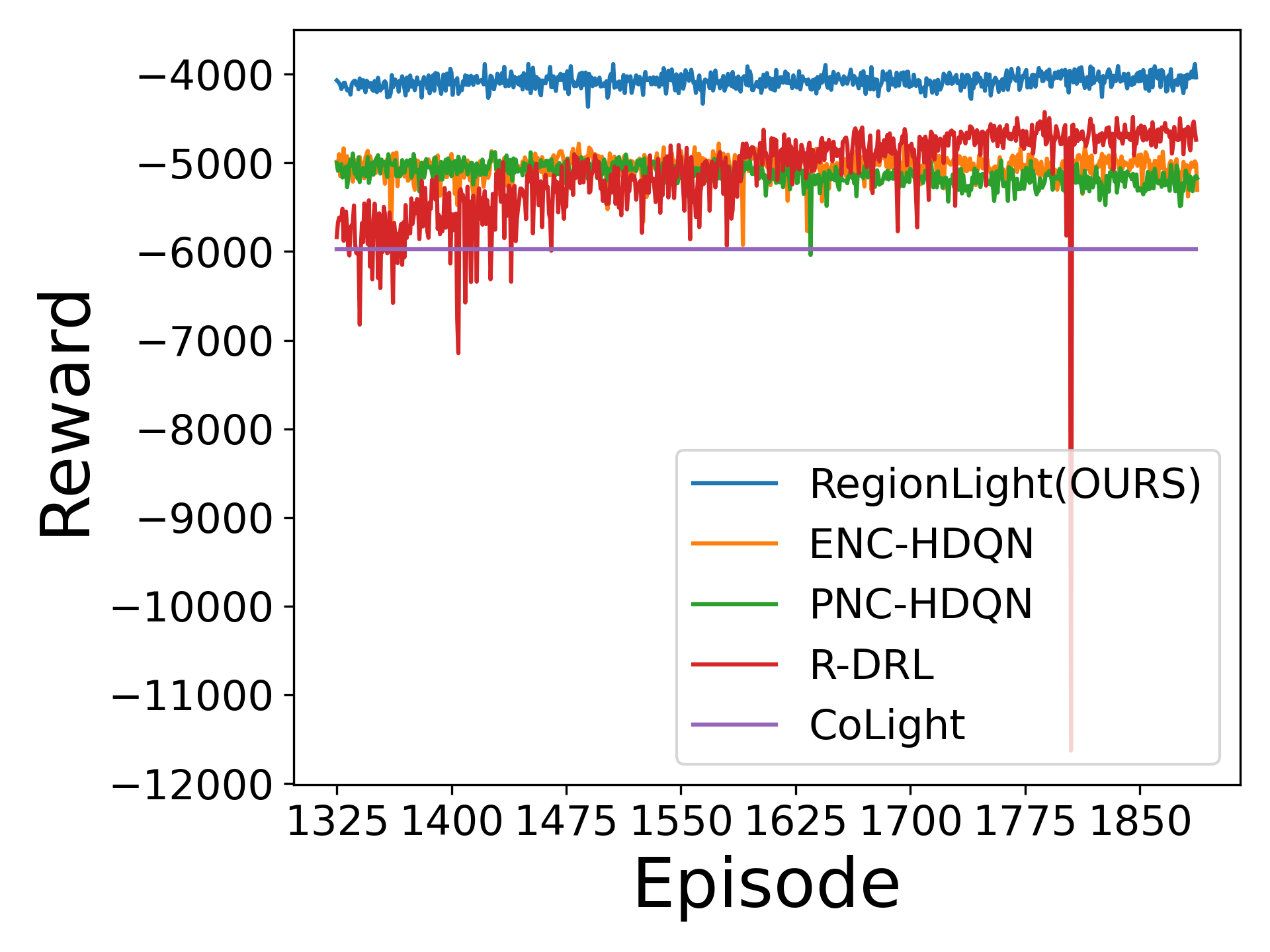}
    }
        \hfill
      \subfloat[Synthetic-2 after Convergence]{
    \includegraphics[width=0.23\textwidth]{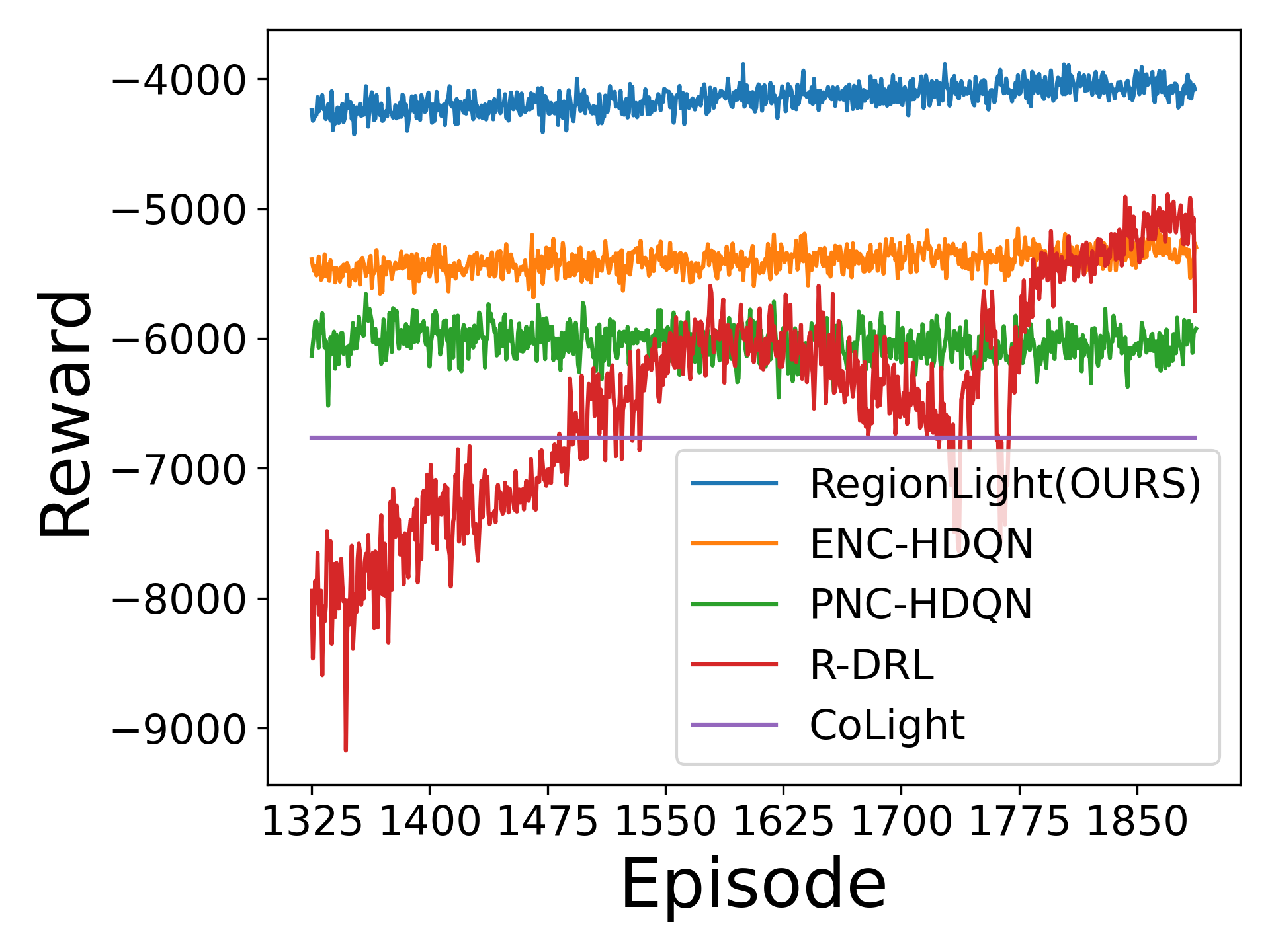}
    }
    \hfill

    \caption{Learning curve of RL agents in $4 \times 4$ networks. The top three pictures are the overall performance during 2000 episodes and the bottom three pictures are the overall performance after episode 1250. Curves of methods failing to converge are omitted.}
    \label{fig:reward curve}
\end{figure*}

\begin{figure*}[htb]
    \centering
      \subfloat[Manhattan]{
    \includegraphics[width=0.3\textwidth]{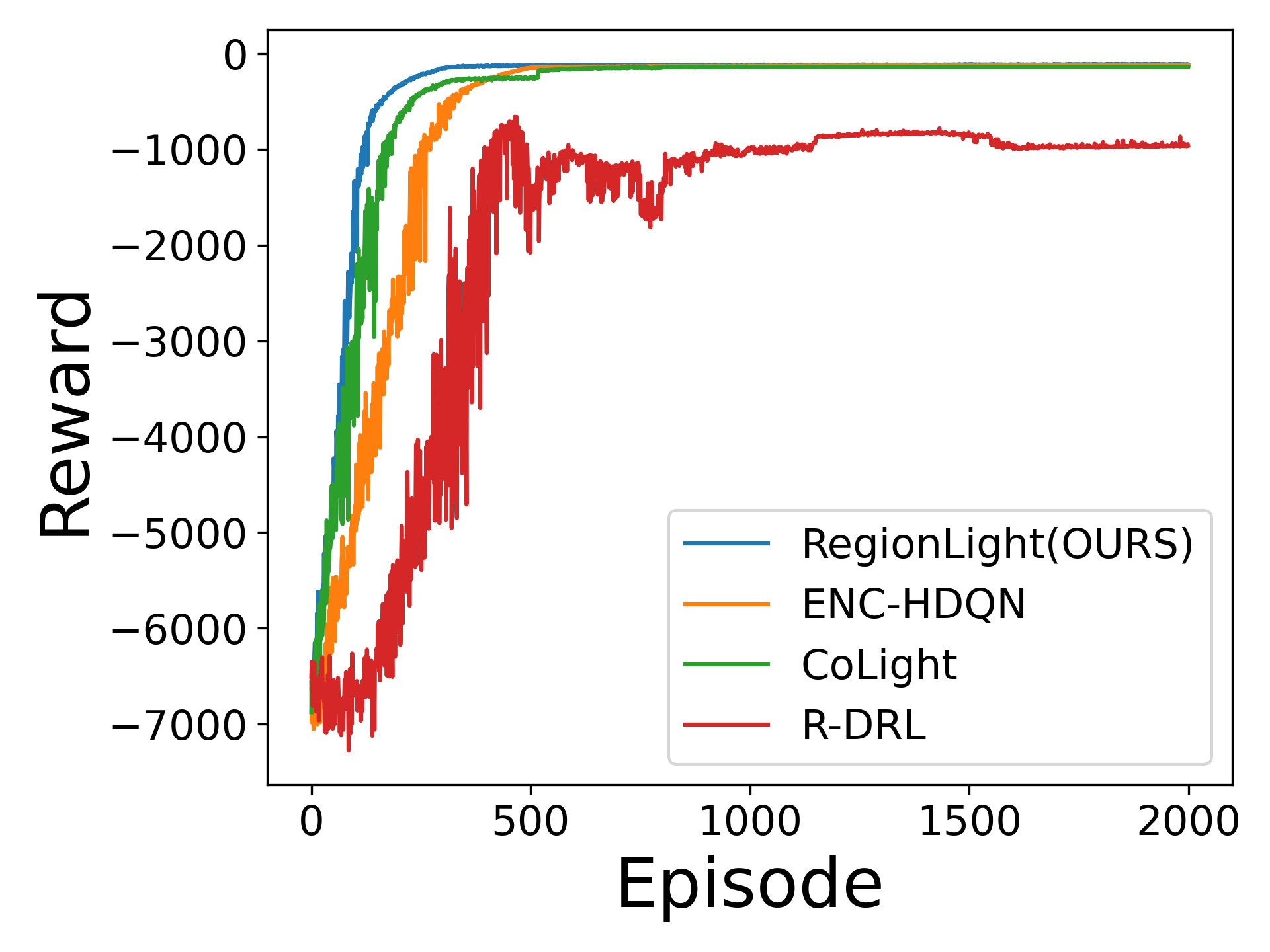}
    }  
      \subfloat[Manhattan after Convergence]{
    \includegraphics[width=0.3\textwidth]{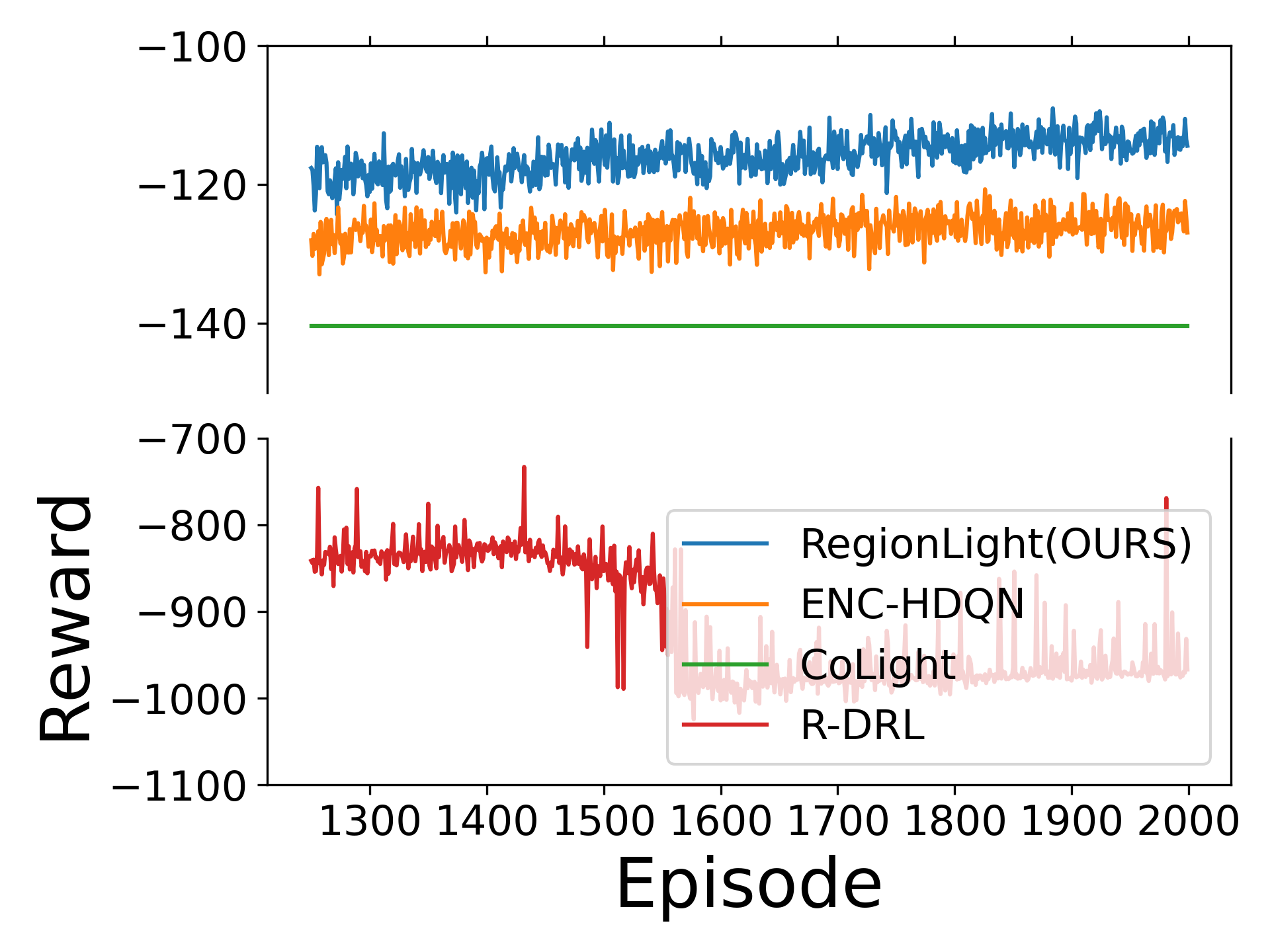}
    }
    \caption{Learning curve of RL agents in $16 \times 3$ networks. The left picture is the overall performance during 2000 episodes and the right picture is the overall performance after episode 1250. Curves of methods failing to converge are omitted.}
    \label{fig:reward curve 16_3 network}
\end{figure*}
% \textcolor{red}{

% }
\begin{table*}[htb]
    \centering
        \caption{Numerical Statistics for RL Baselines}
    \begin{tabular}{|c|c|c|c|c|c|c|c|c|}
    \hline
         \multirow{3}{*}{Flow}& \multirow{3}{*}{Metric}& \multirow{3}{*}{NeighbourRL}& \multirow{3}{*}{R-DRL}& \multirow{3}{*}{CoLight}& \multirow{3}{*}{PNC-HDQN}& \multirow{3}{*}{ENC-HDQN}& RegionLight& RegionLight \\
         &&&&&&&(OURS)&(OURS)\\
         &&&&&&&$\gamma=0.9$&$\gamma=0.99$\\
         \hline
         \multirow{3}{*}{Flat}&
         ATT&379.83$\pm$8.37&335.29$\pm$1.65&334.01$\pm$0.95&328.13$\pm$0.49&326.23$\pm$0.59&319.14$\pm$0.43&319.28$\pm$0.43\\
         &AQL&0.29$\pm$0.031&0.13$\pm$0.01&0.28$\pm$0.34&0.10$\pm$0.0020&0.09$\pm$0.0021&0.07$\pm$0.0016&0.07$\pm$0.0017\\
        &TP&2930.79$\pm$7.19&2950.75$\pm$3.81&2899.92$\pm$144.7&2942.9$\pm$3.09&2959.776$\pm$1.60&2963.27$\pm$0.89&2963.75$\pm$0.88\\

         \hline
         \multirow{3}{*}{Peak}&
         ATT&675.01$\pm$60.26&415.75$\pm$4.94&463$\pm$7.11&437.09$\pm$5.84&479.25$\pm$26.14&402.02$\pm$3.00&402.93$\pm$3.27\\
        &AQL&2.17$\pm$0.21&0.58$\pm$0.04&1.31$\pm$0.37&0.75$\pm$0.03&1.12$\pm$0.23&0.44$\pm$0.01&0.45$\pm$0.01\\
        &TP&5669.4$\pm$226.92&6340.18$\pm$18.28&6034$\pm$370.6&6296.85$\pm$29.96&6222.2$\pm$136.89&6382.19$\pm$10.9&6385.8$\pm$11.26\\

         \hline
         \multirow{3}{*}{Synthetic-1}&
         ATT&-&216.03$\pm$7.28&246.14$\pm$18.19&228.24$\pm$2.0&224.6$\pm$1.6&206.6$\pm$1.3&207.60$\pm$1.2\\
        &AQL&-&0.99$\pm$0.1&1.27$\pm$0.2&1.08$\pm$0.02&1.05$\pm$0.02&0.85$\pm$0.016&0.86$\pm$0.013\\
        &TP&-&11173.29$\pm$63.63&11179.2$\pm$201.3&11179.8$\pm$15.67&11215.41$\pm$4.96&11227.93$\pm$1.07&11227.056$\pm$0.94\\
\hline
         \multirow{3}{*}{Synthetic-2}&
             ATT&-&411.33$\pm$9.64&422$\pm$12.9&415.42$\pm$1.91&404.82$\pm$1.39&381.38$\pm$1.20&382.82$\pm$1.26\\
        &AQL&-&1.3$\pm$0.3&1.27$\pm$0.06&1.26$\pm$0.023&1.11$\pm$0.016&0.85$\pm$0.015&0.87$\pm$0.016\\
        &TP&-&10919.38$\pm$83.08&10896.47$\pm$33.66&10886.39$\pm$15.45&10965.65$\pm$9.81&11033.78$\pm$8.75&11029.40$\pm$10.76\\
         \hline
                  \multirow{3}{*}{Manhattan}&
         ATT&-&319.82$\pm$4.34&184.25$\pm$2.437 &-&181.25$\pm$0.372&176.46$\pm$0.316&177.75 $\pm$0.415\\
        &AQL&-&0.2$\pm$0.005&0.047$\pm$0.0024&-&0.026$\pm$0.0005&0.024$\pm$0.0004&0.025$\pm$0.0004\\
        &TP&-&2597.12$\pm$7.81&2824$\pm$0.0&-&2823.996$\pm$0.063&2824$\pm$0.0&2824$\pm$0.0\\

         \hline
    \end{tabular}

    \label{tab:performance evaluation}
\end{table*}
We first compare our model with non-RL baselines. Numerical statistics results are listed in Table \ref{tab:performance evaluation non_rl}. The results of the Artificial Bee Colony in Manhattan scenarios and those of the genetic algorithm in Synthetic-1 and Manhattan scenarios are omitted because these algorithms failed to converge after random initialization. From Table \ref{tab:performance evaluation non_rl}, we can observe that our model outperforms other baselines in all metrics

Then we compare our model with RL baselines. Since the definitions of the reward of RL baselines involve augmentations, we computed the episode reward as the below equation for consistency, i.e.,
\begin{equation}
    R= \frac{1}{|\mathcal{V}|} \sum_{t=1}^\mathcal{T}\sum_v r_v^t
\end{equation}
The learning process of RL models involves lots of episodes. Therefore, we plot the learning curve of episode reward of all RL baselines. Since NeighbourRL failed to converge in the Synthetic and Manhattan datasets and PNC-HDQN failed to converge in the Manhattan dataset, the corresponding curves are omitted. Meanwhile, the performance of our model under different parameter sets is almost identical. Therefore, we only plot the learning curve of our model under $\gamma=0.9$. As illustrated in Fig. \ref{fig:reward curve} and \ref{fig:reward curve 16_3 network}, our agent converges in all scenarios. 
As the arrival rate increases, the curve becomes oscillated. However, our agent is the most stable one and with the smallest fluctuation.  
From the bottom plots in Fig. \ref{fig:reward curve}, our agent achieves the highest reward after convergence. 

Numerical results are listed in Table \ref{tab:performance evaluation} and \ref{tab:performance evaluation non_rl} where the results of models that fail to converge are omitted. We first look at $4\times 4$ networks, i.e., Hangzhou(Flat), Hangzhou(Peak), Synthetic-1 and Synthetic-2. 
In the Hangzhou scenario, as the flow density increases, the average travel time becomes longer and the network becomes congested. 
In two synthetic scenarios, the only difference is the distance between two intersections. Therefore, vehicles travels longer distances and the average travel time is much longer. Moreover, the throughput of Synthetic-2 is slightly smaller than that of Synthetic-1 because, in Synthetic-2, the remaining time is not enough for those vehicles which depart near the end of episode. 
% In the synthetic-1 scenario, the average travel time of vehicles is the shortest and the average queue length is the highest because the intersections in the synthetic-1 network are much closer than those in the Hangzhou network. 
The average queue length of regional agents in the both synthetic scenarios almost doubles that in the Hangzhou peak scenario while those of CoLight and ENC even decrease. This is probably caused by two factors: The first is the distance between intersections and the arrival rate;
If intersections are close, the travel time between intersections becomes much shorter and it is harder to decrease the queue length. The second is that regional agents try to minimize the total queue length in a region and some intersections may have to sacrifice individual rewards to achieve better regional rewards. 
In the Manhattan scenario, some baselines fail to converge probably due to more complex traffic dynamics in the huge network even though the flow of Manhattan is not as demanding as Hangzhou scenarios.
The difference in throughput between agents is not significant because simulation steps are large enough for most vehicles to arrive at their destinations. Overall, our agent achieves the best results and the smallest standard deviation among all metrics. 
\subsection{Robustness of Our Region}
\begin{figure*}[ht]
    \centering
    \subfloat[(Hangzhou)Flat]{
    \includegraphics[width=0.23\textwidth]{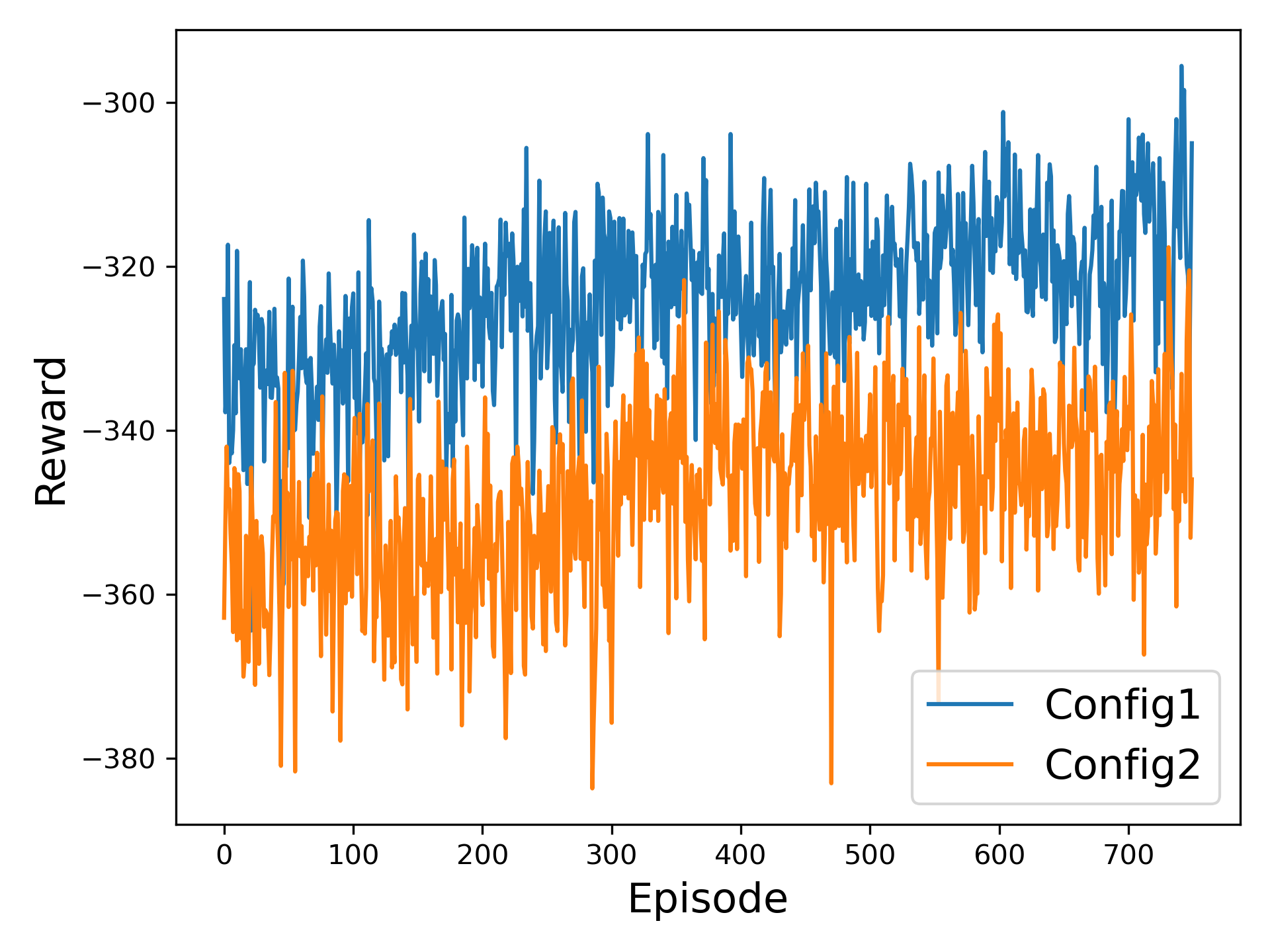}
}
    \hfill
    \subfloat[(Hangzhou)Peak]{
    \includegraphics[width=0.23\textwidth]{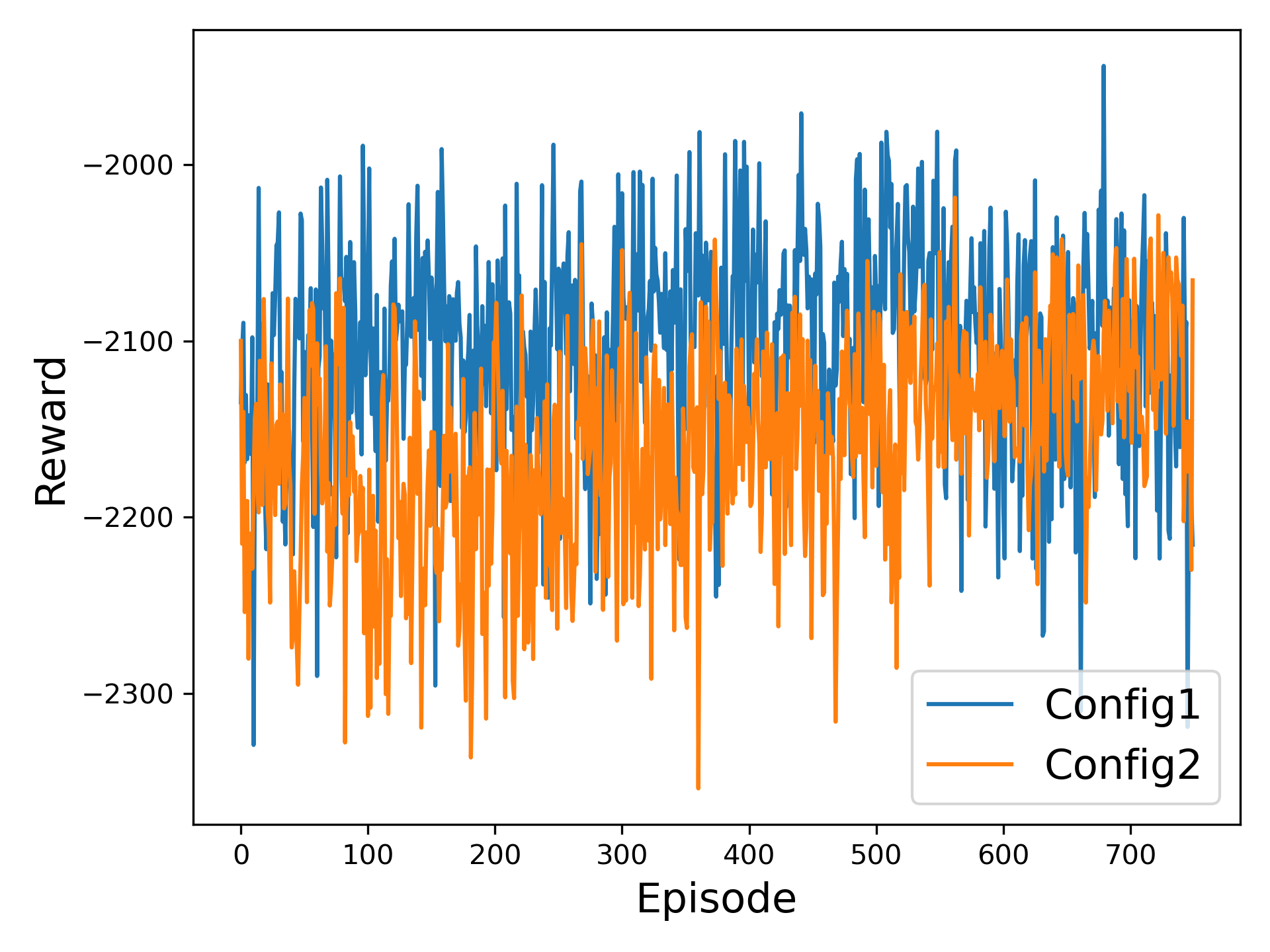}
    }
    \hfill
      \subfloat[Synthetic-1]{
    \includegraphics[width=0.23\textwidth]{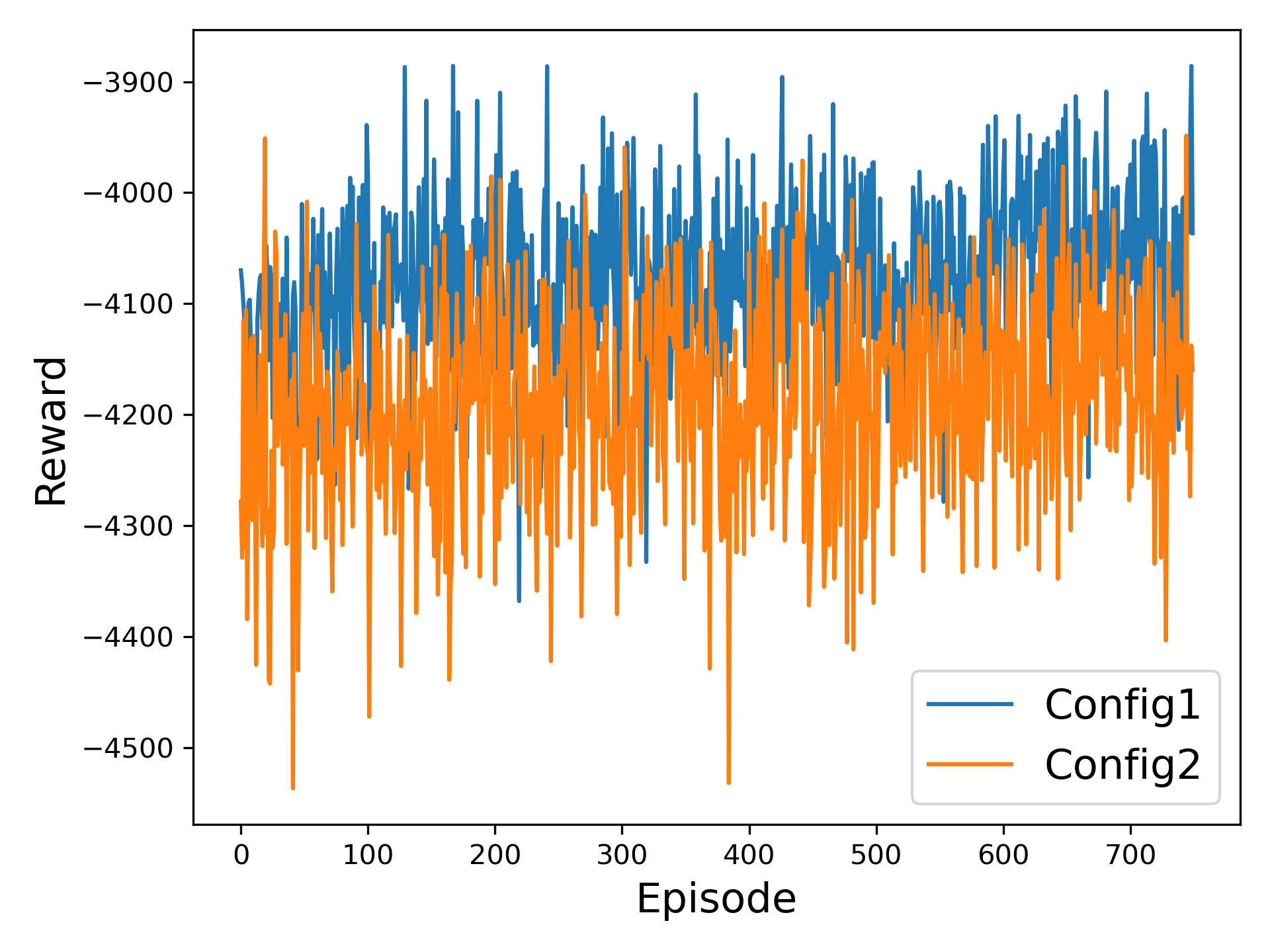}
    }
        \hfill
      \subfloat[Synthetic-2]{
    \includegraphics[width=0.23\textwidth]{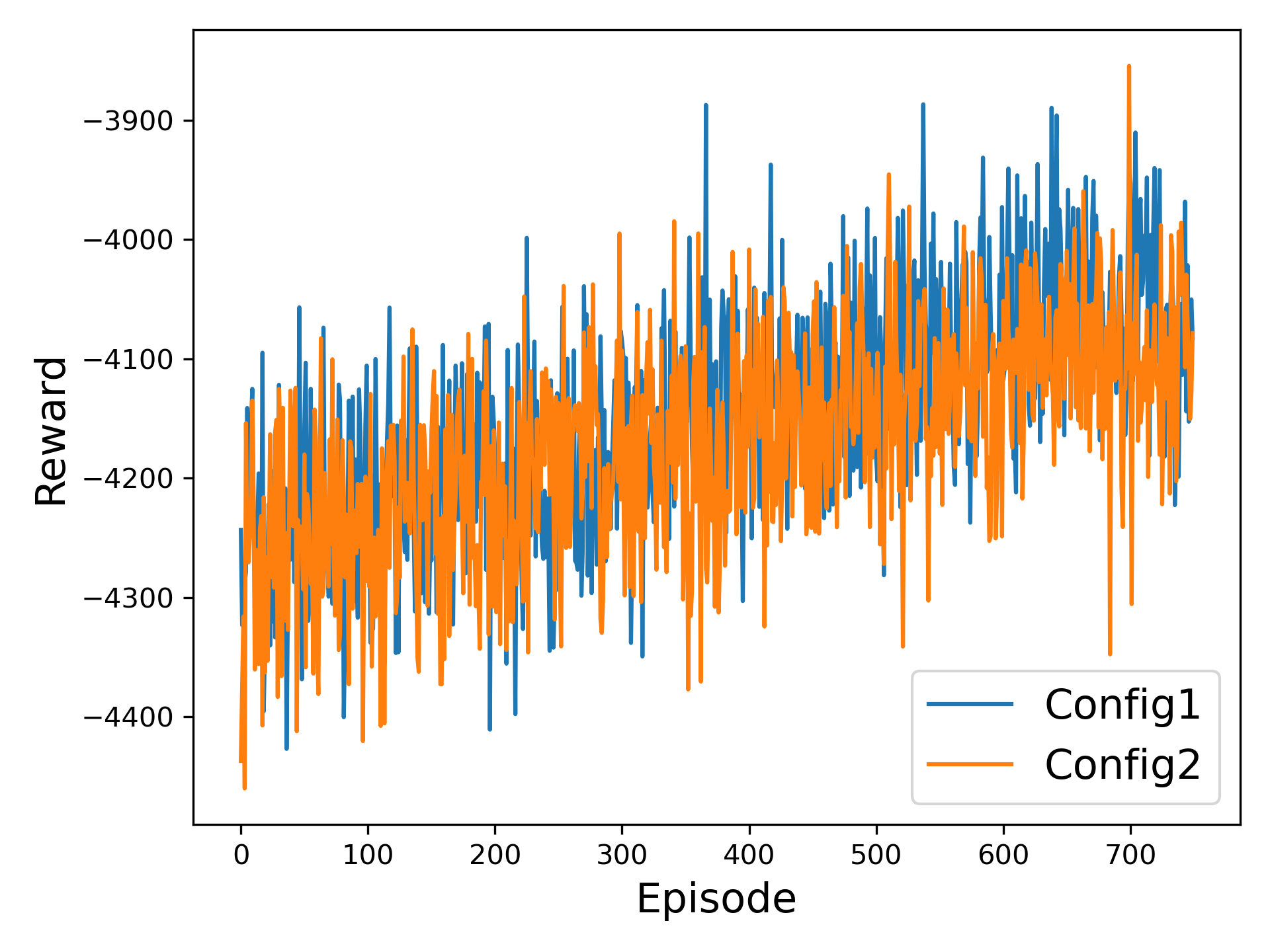}
    }
    \caption{Learning curves of two region configurations in $4\times 4$ grid.}
    \label{fig:reward curve two config}
\end{figure*}
As illustrated in Fig. \ref{fig:partition example}, we generate the unique region configuration based on one minimum dominating set and only one intersection in the region is fictitious. However, one graph may have different minimum dominating sets which generate new sets of region configuration. For the $4\times 4$ grid network, there are two minimum dominating sets. The region configuration generated by first minimum dominating set is  $I_{1-3}, I_{2-1}, I_{3-4}, I_{4-2}$ (Config1) in Fig. \ref{fig:partition example}. The other unique region configuration under the other minimum dominating set is $I_{1-2}, I_{2-4}, I_{3-1}, I_{4-3}$ (Config2). 
% This balance produces two configurations of regions---one is $I_{13}, I_{21}, I_{34}, I_{42}$ (Config1) in Figure \ref{fig:partition example} and the other configuration is  
To ensure the completeness of the experiment, we compare the performance of agents under both configurations.
The learning curves are plotted in Fig. \ref{fig:reward curve two config}.
Agents under both configurations converge but Config2 has a lower converged performance. These training curves indicate that the configuration of regions can affect the training process of agents. In different configurations, agents observe different sets of intersections and traffic flows of these intersections are different. 
The numerical results are listed in Table \ref{tab:performance evaluation two config}. We can observe that there is no significant difference and the numerical results of Config2 are still the best among all baselines.
\begin{table}[h]
    \centering
        \caption{Numerical Statistics of Two Configurations}
    \begin{tabular}{|c|c|c|c|}
    \hline
         Flow&Metric&Config1&Config2 \\
         \hline
         \multirow{3}{*}{Flat}&
         ATT&319.14$\pm$0.43&320.62$\pm$0.48\\&
         AQL&0.07$\pm$0.0016&0.07$\pm$0.0019\\&
         TP&2963.27$\pm$0.89&2963.03$\pm$0.81\\

         \hline
         \multirow{3}{*}{Peak}&
         ATT&402.02$\pm$3.00&406.99$\pm$4.67\\
        &AQL&0.438$\pm$0.01&0.443$\pm$0.009\\
        &TP&6382.19$\pm$10.9&6369.724$\pm$15.2\\

         \hline
         \multirow{3}{*}{Synthetic-1}&
         ATT&206.6$\pm$1.3&208.27$\pm$1.27\\
        &AQL&0.85$\pm$0.0157&0.86$\pm$0.0162\\
        &TP&11227.93$\pm$1.07&11227.41$\pm$ 1.34\\
    \hline
         \multirow{3}{*}{Synthetic-2}&
         ATT&381.38$\pm$1.20&381.80$\pm$1.2\\
        &AQL&0.85$\pm$0.015&0.86$\pm$0.015\\
        &TP&11033.78$\pm$8.75&11030.73$\pm$8.84\\
         \hline
    \end{tabular}

    \label{tab:performance evaluation two config}
\end{table}

If the minimum dominant set does not meet Theorem \ref{theorem: uniqueness}, then different iteration orders of line 13 in Algorithm \ref{alg:Construct} can construct different sets of region configuration according to Remark \ref{remark: reconstruct}. We tried several assignment orders to investigate whether different assignment orders could affect performance. From Table \ref{tab:performance evaluation two config}, even for configurations constructed from different minimum dominating sets, the performance of models under those configurations is very close. Therefore, we listed the results of worst performance, average performance, and best performance after convergence. From Table \ref{tab:performance evaluation assignment order}, numerical statistics indicate that, although performance under different assignment orders is not identical, the difference is very close and performance is consistent. 
\begin{table}[h]
    \centering
        \caption{Numerical Statistics of Different Assignment Orders in Manhattan Scenario}
    \begin{tabular}{|c|c|c|c|}
    \hline
         &ATT&AQL&TP \\
         \hline
        Best Performance&176.03$\pm$0.501&0.0241$\pm$0.0412&2824$\pm$0.0\\
         
         \hline
         Average Performance&177.64$\pm$0.684&0.0245$\pm$0.0415&2824$\pm$0.0\\
         \hline
         Worst Performance&178.25$\pm$0.641&0.0249$\pm$0.0419&2824$\pm$0.0\\
         \hline
    \end{tabular}

    \label{tab:performance evaluation assignment order}
\end{table}
\subsection{Improvement of ABDQ over BDQ and DDPG+WA}
The computation of target values in ABDQ only involves activated branches to relieve the negative influence of fictitious intersections.
\begin{table}[h]
    \centering
        \caption{Improvement by ABDQ }
    \begin{tabular}{|c|c|c|c|}
    \hline
         Flow&Metric& Config1&$2\times 3$ Grid\\
         \hline
         \multirow{3}{*}{Flat}&
         ATT&2.8\%&3.5\%\\&
         AQL&9.7\%&24.1\%\\&
         TP&0.14\%&0.27\% \\
         \hline
         \multirow{3}{*}{Peak}&
         ATT&1.1\%&1.8\%\\
        &AQL&7.8\%&19.3\%\\
        &TP&0.12\%&0.42\%\\
         \hline
         \multirow{3}{*}{Synthetic-1}&
         ATT&1.7\%&2.4 \%\\
        &AQL&8.5\%&11.8\%\\
        &TP&0.27\%&0.34\%\\
         \hline
        \multirow{3}{*}{Synthetic-2}&
         ATT&2.3\%&5.2\%\\
        &AQL&9.4\%&18.9\%\\
        &TP&0.35\%&1.3\%\\
         \hline
    \end{tabular}
    \label{tab:Improvement by DBDQ}
\end{table}
 To show the advantage of ABDQ, we applied BDQ on "Config1" and ABDQ on $2 \times 3$ grid regions. "Config1" has one fictitious intersection and the $2 \times 3$ grid region has two fictitious intersections. Both regions have four activated branches in BDQ and ABDQ. In Table \ref{tab:Improvement by DBDQ}, the performance in both region configurations is improved by ABDQ. For different scenarios, the improvement in the Hangzhou flat scenario is more significant than that in other scenarios. For different region configurations, the improvement of the grid region is more significant than "Config1".
 It is probably because there are more fictitious intersections in grid regions and ABDQ alleviates the negative of fictitious intersections effectively.
\section{Conclusion and Future Work}
In this paper, we proposed a novel regional signal control framework for TSC. Our network partitioning rule has the potential to be applied in non-grid networks because the topology of our region is a star which is composed of one center and an arbitrary number of leaves. Meanwhile, we proposed ABDQ to search for optimal joint action and to mitigate the negative influence of fictitious intersections. There are three major advantages of ABDQ: One is that the search for optimal actions is more efficient because the output size of the neural network grows linearly as the size of the region increases. The second one is intra-regional cooperation control. The last one is that ABDQ mitigates the influence of fictitious intersections by calculating target value and loss based on activated branches.

We have carried out comprehensive experiments to evaluate the performance and robustness of our RL agent. Experiments show that our RL agent achieves the best performance in both real and synthetic scenarios, especially in scenarios with a high-density flow. Also, we observed that the distance between intersections can affect the convergence rate of agents and their performance. Interestingly, different configurations and assignment orders can influence the performance of agents, but the difference is not significant.
One limitation of our framework is that we model regional RL agents as independent learners and no explicit cooperation is designed between regions.

In the future, we will focus on the maximum capacity of our region by increasing hops when defining the neighborhood of one intersection. It is also worthwhile to investigate the performance of our model on 8-phases traffic signals and in massive-scale or non-grid traffic networks. Meanwhile, region-wise cooperation is also a promising research direction.
\label{sec: conclusion}
\bibliographystyle{IEEEtran}
\bibliography{reference}

\begin{thebibliography}{10}
\providecommand{\url}[1]{#1}
\csname url@samestyle\endcsname
\providecommand{\newblock}{\relax}
\providecommand{\bibinfo}[2]{#2}
\providecommand{\BIBentrySTDinterwordspacing}{\spaceskip=0pt\relax}
\providecommand{\BIBentryALTinterwordstretchfactor}{4}
\providecommand{\BIBentryALTinterwordspacing}{\spaceskip=\fontdimen2\font plus
\BIBentryALTinterwordstretchfactor\fontdimen3\font minus
  \fontdimen4\font\relax}
\providecommand{\BIBforeignlanguage}[2]{{%
\expandafter\ifx\csname l@#1\endcsname\relax
\typeout{** WARNING: IEEEtran.bst: No hyphenation pattern has been}%
\typeout{** loaded for the language `#1'. Using the pattern for}%
\typeout{** the default language instead.}%
\else
\language=\csname l@#1\endcsname
\fi
#2}}
\providecommand{\BIBdecl}{\relax}
\BIBdecl

\bibitem{inrix2021uk}
\BIBentryALTinterwordspacing
``Inrix 2021 global traffic scorecard: As lockdowns ease uk city centres show
  signs of return to 2019 levels of congestion,'' 2022. [Online]. Available:
  \url{https://inrix.com/press-releases/2021-traffic-scorecard-uk/}
\BIBentrySTDinterwordspacing

\bibitem{inrix2021us}
\BIBentryALTinterwordspacing
``Inrix: Americans lost 3.4 billion hours due to congestion in 2021, 42\% below
  pre-covid,'' 2022. [Online]. Available:
  \url{https://inrix.com/press-releases/2021-traffic-scorecard-uk/}
\BIBentrySTDinterwordspacing

\bibitem{who}
\BIBentryALTinterwordspacing
``who releases country estimates on air pollution exposure and health
  impact\_2022,'' 2022. [Online]. Available:
  \url{www.who.int/news/item/27-09-2016-who-releases-country-estimates-on-air-pollution-exposure-and-health-impact}
\BIBentrySTDinterwordspacing

\bibitem{roess2004traffic}
R.~P. Roess, E.~S. Prassas, and W.~R. McShane, \emph{Traffic
  engineering}.\hskip 1em plus 0.5em minus 0.4em\relax Pearson/Prentice Hall,
  2004.

\bibitem{webster1958traffic}
F.~V. Webster, ``Traffic signal settings,'' Tech. Rep., 1958.

\bibitem{koonce2008traffic}
P.~Koonce and L.~Rodegerdts, ``Traffic signal timing manual.'' United States.
  Federal Highway Administration, Tech. Rep., 2008.

\bibitem{little1981maxband}
J.~D. Little, M.~D. Kelson, and N.~H. Gartner, ``Maxband: A versatile program
  for setting signals on arteries and triangular networks,'' 1981.

\bibitem{varaiya2013max}
P.~Varaiya, ``Max pressure control of a network of signalized intersections,''
  \emph{Transportation Research Part C: Emerging Technologies}, vol.~36, pp.
  177--195, 2013.

\bibitem{gershenson2004self}
C.~Gershenson, ``Self-organizing traffic lights,'' \emph{arXiv preprint
  nlin/0411066}, 2004.

\bibitem{cools2013self}
S.-B. Cools, C.~Gershenson, and B.~D’Hooghe, ``Self-organizing traffic
  lights: A realistic simulation,'' \emph{Advances in applied self-organizing
  systems}, pp. 45--55, 2013.

\bibitem{long2022traffic}
G.~Long, A.~Wang, and T.~Jiang, ``Traffic signal self-organizing control with
  road capacity constraints,'' \emph{IEEE Transactions on Intelligent
  Transportation Systems}, vol.~23, no.~10, pp. 18\,502--18\,511, 2022.

\bibitem{shaikh2020review}
P.~W. Shaikh, M.~El-Abd, M.~Khanafer, and K.~Gao, ``A review on swarm
  intelligence and evolutionary algorithms for solving the traffic signal
  control problem,'' \emph{IEEE transactions on intelligent transportation
  systems}, vol.~23, no.~1, pp. 48--63, 2020.

\bibitem{gao2018solving}
K.~Gao, Y.~Zhang, R.~Su, F.~Yang, P.~N. Suganthan, and M.~Zhou, ``Solving
  traffic signal scheduling problems in heterogeneous traffic network by using
  meta-heuristics,'' \emph{IEEE Transactions on Intelligent Transportation
  Systems}, vol.~20, no.~9, pp. 3272--3282, 2018.

\bibitem{mitchell1998introduction}
M.~Mitchell, \emph{An introduction to genetic algorithms}.\hskip 1em plus 0.5em
  minus 0.4em\relax MIT press, 1998.

\bibitem{karaboga2005idea}
D.~Karaboga \emph{et~al.}, ``An idea based on honey bee swarm for numerical
  optimization,'' Technical report-tr06, Erciyes university, engineering
  faculty, computer~…, Tech. Rep., 2005.

\bibitem{geem2001new}
Z.~W. Geem, J.~H. Kim, and G.~V. Loganathan, ``A new heuristic optimization
  algorithm: harmony search,'' \emph{simulation}, vol.~76, no.~2, pp. 60--68,
  2001.

\bibitem{gao2018meta}
K.~Gao, Y.~Zhang, Y.~Zhang, R.~Su, and P.~N. Suganthan, ``Meta-heuristics for
  bi-objective urban traffic light scheduling problems,'' \emph{IEEE
  transactions on intelligent transportation systems}, vol.~20, no.~7, pp.
  2618--2629, 2018.

\bibitem{wang2022problem}
L.~Wang, K.~Gao, Z.~Lin, and W.~Huang, ``Problem feature-based meta-heuristics
  with reinforcement learning for solving urban traffic light scheduling
  problems,'' in \emph{2022 IEEE 25th International Conference on Intelligent
  Transportation Systems (ITSC)}.\hskip 1em plus 0.5em minus 0.4em\relax IEEE,
  2022, pp. 845--850.

\bibitem{tan2018hierarchical}
M.~K. Tan, H.~S.~E. Chuo, R.~K.~Y. Chin, K.~B. Yeo, and K.~T.~K. Teo,
  ``Hierarchical multi-agent system in traffic network signalization with
  improved genetic algorithm,'' in \emph{2018 IEEE International Conference on
  Artificial Intelligence in Engineering and Technology (IICAIET)}.\hskip 1em
  plus 0.5em minus 0.4em\relax IEEE, 2018, pp. 1--6.

\bibitem{gao2019meta}
K.~Gao, N.~Wu, and R.~Wang, ``Meta-heuristic and milp for solving urban traffic
  signal control,'' in \emph{2019 International Conference on Industrial
  Engineering and Systems Management (IESM)}.\hskip 1em plus 0.5em minus
  0.4em\relax IEEE, 2019, pp. 1--5.

\bibitem{sutton2018reinforcement}
R.~S. Sutton and A.~G. Barto, \emph{Reinforcement learning: An
  introduction}.\hskip 1em plus 0.5em minus 0.4em\relax MIT press, 2018.

\bibitem{mnih2015human}
V.~Mnih, K.~Kavukcuoglu, D.~Silver, A.~A. Rusu, J.~Veness, M.~G. Bellemare,
  A.~Graves, M.~Riedmiller, A.~K. Fidjeland, G.~Ostrovski \emph{et~al.},
  ``Human-level control through deep reinforcement learning,'' \emph{nature},
  vol. 518, no. 7540, pp. 529--533, 2015.

\bibitem{kober2013reinforcement}
J.~Kober, J.~A. Bagnell, and J.~Peters, ``Reinforcement learning in robotics: A
  survey,'' \emph{The International Journal of Robotics Research}, vol.~32,
  no.~11, pp. 1238--1274, 2013.

\bibitem{vinyals2019alphastar}
O.~Vinyals, I.~Babuschkin, J.~Chung, M.~Mathieu, M.~Jaderberg, W.~M. Czarnecki,
  A.~Dudzik, A.~Huang, P.~Georgiev, R.~Powell \emph{et~al.}, ``Alphastar:
  Mastering the real-time strategy game starcraft ii,'' \emph{DeepMind blog},
  vol.~2, 2019.

\bibitem{thorpe1996tra}
T.~L. Thorpe and C.~W. Anderson, ``Traffic light control using sarsa with three
  state representations,'' \emph{Technical report, Citeseer}, 1996.

\bibitem{wen2007stochastic}
K.~Wen, S.~Qu, and Y.~Zhang, ``A stochastic adaptive control model for isolated
  intersections,'' in \emph{2007 IEEE International Conference on Robotics and
  Biomimetics (ROBIO)}.\hskip 1em plus 0.5em minus 0.4em\relax IEEE, 2007, pp.
  2256--2260.

\bibitem{el2010agent}
S.~El-Tantawy and B.~Abdulhai, ``An agent-based learning towards decentralized
  and coordinated traffic signal control,'' in \emph{13th International IEEE
  Conference on Intelligent Transportation Systems}.\hskip 1em plus 0.5em minus
  0.4em\relax IEEE, 2010, pp. 665--670.

\bibitem{haydari2020deep}
A.~Haydari and Y.~Y{\i}lmaz, ``Deep reinforcement learning for intelligent
  transportation systems: A survey,'' \emph{IEEE Transactions on Intelligent
  Transportation Systems}, vol.~23, no.~1, pp. 11--32, 2020.

\bibitem{wei2019survey}
H.~Wei, G.~Zheng, V.~Gayah, and Z.~Li, ``A survey on traffic signal control
  methods,'' \emph{arXiv preprint arXiv:1904.08117}, 2019.

\bibitem{tan1993multi}
M.~Tan, ``Multi-agent reinforcement learning: Independent vs. cooperative
  agents,'' in \emph{Proceedings of the tenth international conference on
  machine learning}, 1993, pp. 330--337.

\bibitem{van2016coordinated}
E.~Van~der Pol and F.~A. Oliehoek, ``Coordinated deep reinforcement learners
  for traffic light control,'' \emph{Proceedings of learning, inference and
  control of multi-agent systems (at NIPS 2016)}, vol.~1, 2016.

\bibitem{zhang2022neighborhood}
C.~Zhang, Y.~Tian, Z.~Zhang, W.~Xue, X.~Xie, T.~Yang, X.~Ge, and R.~Chen,
  ``Neighborhood cooperative multiagent reinforcement learning for adaptive
  traffic signal control in epidemic regions,'' \emph{IEEE Transactions on
  Intelligent Transportation Systems}, vol.~23, no.~12, pp. 25\,157--25\,168,
  2022.

\bibitem{wang2020large}
X.~Wang, L.~Ke, Z.~Qiao, and X.~Chai, ``Large-scale traffic signal control
  using a novel multiagent reinforcement learning,'' \emph{IEEE transactions on
  cybernetics}, vol.~51, no.~1, pp. 174--187, 2020.

\bibitem{zhang2021multi}
K.~Zhang, Z.~Yang, and T.~Ba{\c{s}}ar, ``Multi-agent reinforcement learning: A
  selective overview of theories and algorithms,'' \emph{Handbook of
  reinforcement learning and control}, pp. 321--384, 2021.

\bibitem{chu2016large}
T.~Chu, S.~Qu, and J.~Wang, ``Large-scale traffic grid signal control with
  regional reinforcement learning,'' in \emph{2016 american control conference
  (acc)}.\hskip 1em plus 0.5em minus 0.4em\relax IEEE, 2016, pp. 815--820.

\bibitem{tan2019cooperative}
T.~Tan, F.~Bao, Y.~Deng, A.~Jin, Q.~Dai, and J.~Wang, ``Cooperative deep
  reinforcement learning for large-scale traffic grid signal control,''
  \emph{IEEE transactions on cybernetics}, vol.~50, no.~6, pp. 2687--2700,
  2019.

\bibitem{kok2005using}
J.~R. Kok and N.~Vlassis, ``Using the max-plus algorithm for multiagent
  decision making in coordination graphs,'' in \emph{Robot Soccer World
  Cup}.\hskip 1em plus 0.5em minus 0.4em\relax Springer, 2005, pp. 1--12.

\bibitem{tan2018large}
T.~Tan, T.~Chu, B.~Peng, and J.~Wang, ``Large-scale traffic grid signal control
  using decentralized fuzzy reinforcement learning,'' in \emph{Proceedings of
  SAI Intelligent Systems Conference (IntelliSys) 2016: Volume 1}.\hskip 1em
  plus 0.5em minus 0.4em\relax Springer, 2018, pp. 652--662.

\bibitem{arel2010reinforcement}
I.~Arel, C.~Liu, T.~Urbanik, and A.~G. Kohls, ``Reinforcement learning-based
  multi-agent system for network traffic signal control,'' \emph{IET
  Intelligent Transport Systems}, vol.~4, no.~2, pp. 128--135, 2010.

\bibitem{wei2019presslight}
H.~Wei, C.~Chen, G.~Zheng, K.~Wu, V.~Gayah, K.~Xu, and Z.~Li, ``Presslight:
  Learning max pressure control to coordinate traffic signals in arterial
  network,'' in \emph{Proceedings of the 25th ACM SIGKDD International
  Conference on Knowledge Discovery \& Data Mining}, 2019, pp. 1290--1298.

\bibitem{wei2019colight}
H.~Wei, N.~Xu, H.~Zhang, G.~Zheng, X.~Zang, C.~Chen, W.~Zhang, Y.~Zhu, K.~Xu,
  and Z.~Li, ``Colight: Learning network-level cooperation for traffic signal
  control,'' in \emph{Proceedings of the 28th ACM International Conference on
  Information and Knowledge Management}, 2019, pp. 1913--1922.

\bibitem{wang2021traffic}
M.~Wang, L.~Wu, J.~Li, and L.~He, ``Traffic signal control with reinforcement
  learning based on region-aware cooperative strategy,'' \emph{IEEE
  Transactions on Intelligent Transportation Systems}, vol.~23, no.~7, pp.
  6774--6785, 2021.

\bibitem{jiang2018graph}
J.~Jiang, C.~Dun, T.~Huang, and Z.~Lu, ``Graph convolutional reinforcement
  learning,'' \emph{arXiv preprint arXiv:1810.09202}, 2018.

\bibitem{nishi2018traffic}
T.~Nishi, K.~Otaki, K.~Hayakawa, and T.~Yoshimura, ``Traffic signal control
  based on reinforcement learning with graph convolutional neural nets,'' in
  \emph{2018 21st International conference on intelligent transportation
  systems (ITSC)}.\hskip 1em plus 0.5em minus 0.4em\relax IEEE, 2018, pp.
  877--883.

\bibitem{devailly2021ig}
F.-X. Devailly, D.~Larocque, and L.~Charlin, ``Ig-rl: Inductive graph
  reinforcement learning for massive-scale traffic signal control,'' \emph{IEEE
  Transactions on Intelligent Transportation Systems}, vol.~23, no.~7, pp.
  7496--7507, 2021.

\bibitem{zang2020metalight}
X.~Zang, H.~Yao, G.~Zheng, N.~Xu, K.~Xu, and Z.~Li, ``Metalight: Value-based
  meta-reinforcement learning for traffic signal control,'' in
  \emph{Proceedings of the AAAI Conference on Artificial Intelligence},
  vol.~34, no.~01, 2020, pp. 1153--1160.

\bibitem{jiang2021distributed}
S.~Jiang, Y.~Huang, M.~Jafari, and M.~Jalayer, ``A distributed multi-agent
  reinforcement learning with graph decomposition approach for large-scale
  adaptive traffic signal control,'' \emph{IEEE Transactions on Intelligent
  Transportation Systems}, vol.~23, no.~9, pp. 14\,689--14\,701, 2021.

\bibitem{shi2000normalized}
J.~Shi and J.~Malik, ``Normalized cuts and image segmentation,'' \emph{IEEE
  Transactions on pattern analysis and machine intelligence}, vol.~22, no.~8,
  pp. 888--905, 2000.

\bibitem{littman1994markov}
M.~L. Littman, ``Markov games as a framework for multi-agent reinforcement
  learning,'' in \emph{Machine learning proceedings 1994}.\hskip 1em plus 0.5em
  minus 0.4em\relax Elsevier, 1994, pp. 157--163.

\bibitem{watkins1992q}
C.~J. Watkins and P.~Dayan, ``Q-learning,'' \emph{Machine learning}, vol.~8,
  no.~3, pp. 279--292, 1992.

\bibitem{tavakoli2018action}
A.~Tavakoli, F.~Pardo, and P.~Kormushev, ``Action branching architectures for
  deep reinforcement learning,'' in \emph{Proceedings of the aaai conference on
  artificial intelligence}, vol.~32, no.~1, 2018.

\bibitem{west2001introduction}
D.~B. West \emph{et~al.}, \emph{Introduction to graph theory}.\hskip 1em plus
  0.5em minus 0.4em\relax Prentice hall Upper Saddle River, 2001, vol.~2.

\bibitem{papadimitriou1998combinatorial}
C.~H. Papadimitriou and K.~Steiglitz, \emph{Combinatorial optimization:
  algorithms and complexity}.\hskip 1em plus 0.5em minus 0.4em\relax Courier
  Corporation, 1998.

\bibitem{duraisamy2021linear}
P.~Duraisamy and S.~Esakkimuthu, ``Linear programming approach for various
  domination parameters,'' \emph{Discrete Mathematics, Algorithms and
  Applications}, vol.~13, no.~01, p. 2050096, 2021.

\bibitem{wei2018intellilight}
H.~Wei, G.~Zheng, H.~Yao, and Z.~Li, ``Intellilight: A reinforcement learning
  approach for intelligent traffic light control,'' in \emph{Proceedings of the
  24th ACM SIGKDD International Conference on Knowledge Discovery \& Data
  Mining}, 2018, pp. 2496--2505.

\bibitem{chu2019multi}
T.~Chu, J.~Wang, L.~Codec{\`a}, and Z.~Li, ``Multi-agent deep reinforcement
  learning for large-scale traffic signal control,'' \emph{IEEE Transactions on
  Intelligent Transportation Systems}, vol.~21, no.~3, pp. 1086--1095, 2019.

\bibitem{zheng2019diagnosing}
G.~Zheng, X.~Zang, N.~Xu, H.~Wei, Z.~Yu, V.~Gayah, K.~Xu, and Z.~Li,
  ``Diagnosing reinforcement learning for traffic signal control,'' \emph{arXiv
  preprint arXiv:1905.04716}, 2019.

\bibitem{zhang2019cityflow}
H.~Zhang, S.~Feng, C.~Liu, Y.~Ding, Y.~Zhu, Z.~Zhou, W.~Zhang, Y.~Yu, H.~Jin,
  and Z.~Li, ``Cityflow: A multi-agent reinforcement learning environment for
  large scale city traffic scenario,'' in \emph{The world wide web conference},
  2019, pp. 3620--3624.

\bibitem{gurobi}
\BIBentryALTinterwordspacing
{Gurobi Optimization, LLC}, ``{Gurobi Optimizer Reference Manual},'' 2023.
  [Online]. Available: \url{https://www.gurobi.com}
\BIBentrySTDinterwordspacing

\bibitem{abadi2016tensorflow}
M.~Abadi, P.~Barham, J.~Chen, Z.~Chen, A.~Davis, J.~Dean, M.~Devin,
  S.~Ghemawat, G.~Irving, M.~Isard \emph{et~al.}, ``$\{$TensorFlow$\}$: a
  system for $\{$Large-Scale$\}$ machine learning,'' in \emph{12th USENIX
  symposium on operating systems design and implementation (OSDI 16)}, 2016,
  pp. 265--283.

\end{thebibliography}
\end{document}